\documentclass{article}

\PassOptionsToPackage{numbers}{natbib}

     \usepackage[final]{neurips_2020}

\usepackage[utf8]{inputenc} 
\usepackage[T1]{fontenc}    
\usepackage{hyperref}       
\usepackage{url}            
\usepackage{booktabs}       
\usepackage{amsfonts}       
\usepackage{nicefrac}       
\usepackage{microtype}      

\usepackage{amsmath}
\usepackage{color}
\usepackage{enumitem}
\usepackage{amssymb}
\usepackage{amsthm}
\usepackage{graphicx}

\newcommand{\alphalow}{{\underline{\alpha}}}
\newcommand{\alphahigh}{{\overline{\alpha}}}

\DeclareMathOperator*{\argmin}{argmin}

\DeclareMathOperator{\Tr}{Tr}

\DeclareMathOperator{\Id}{Id}

\definecolor{NavyBlue}{rgb}{0.1,0.1,0.6}

\def\bfone{{\boldsymbol 1}}

\def\var{{\rm var\,}}
\newcommand{\diff}{\mathrm{d}}

\newcommand{\N}{\mathbb{N}}
\newcommand{\Z}{\mathbb{Z}}
\newcommand{\E}{\mathbb{E}}
\newcommand{\R}{\mathbb{R}}
\newcommand{\T}{\mathbb{T}}

\newcommand{\cN}{\mathcal{N}}
\newcommand{\cE}{\mathcal{E}}

\newcommand{\cH}{\mathcal{H}}
\newcommand{\cV}{\mathcal{V}}

\newcommand{\cR}{\mathcal{R}}
\newcommand{\cU}{\mathcal{U}}

\renewcommand{\leq}{\leqslant}
\renewcommand{\geq}{\geqslant}

\newtheorem{remark}{Remark}

\newtheorem{property}{Property}
\newtheorem{proposition}{Proposition}
\newtheorem{thm}{Theorem}
\newtheorem{lem}{Lemma}
\newtheorem{coro}{Corollary}

\title{Tight~Nonparametric~Convergence~Rates for~Stochastic~Gradient~Descent under~the~Noiseless~Linear~Model}

\author{
	Rapha\"el Berthier \\
	INRIA, Ecole Normale Supérieure\\
PSL Research University, Paris, France\\
	\texttt{raphael.berthier@inria.fr} \\
	\And
	Francis Bach \\
	INRIA, Ecole Normale Supérieure\\
PSL Research University, Paris, France\\
	\texttt{francis.bach@inria.fr} \\
	\And
	Pierre Gaillard \\
	INRIA, Ecole Normale Supérieure\\
PSL Research University, Paris, France\\
	\texttt{pierre.gaillard@inria.fr} \\
}

\begin{document}

\maketitle

\begin{abstract}
	In the context of statistical supervised learning, the noiseless linear model assumes that there exists a deterministic linear relation $Y = \langle \theta_*, \Phi(U) \rangle$ between the random output $Y$ and the random feature vector $\Phi(U)$, a potentially non-linear transformation of the inputs~$U$. 
 We analyze the convergence of single-pass, fixed step-size stochastic gradient descent on the least-square risk under this model. The convergence of the iterates to the optimum $\theta_*$ and the decay of the generalization error follow polynomial convergence rates with exponents that both depend on the regularities of the optimum $\theta_*$ and of the feature vectors $\Phi(U)$. We interpret our result in the reproducing kernel Hilbert space framework. As a special case, we analyze an online algorithm for estimating a real function on the unit hypercube from the noiseless observation of its value at randomly sampled points; the convergence depends on the Sobolev smoothness of the function and of a chosen kernel. Finally, we apply our analysis beyond the supervised learning setting to obtain convergence rates for the averaging process (a.k.a.~gossip algorithm) on a graph depending on its spectral dimension.
\end{abstract}

\section{Introduction}
\label{sec:intro}

Linear regression is widely used in statistical supervised learning, sometimes in the implicit form of kernel regression. A large theory describes the performance (reconstruction and generalization errors) of various algorithms (penalized least-squares, stochastic gradient descent, \dots) under various data models (e.g., noisy or noiseless linear model) and the corresponding minimax bounds. In nonparametric estimation theory, one seeks bounds independent of the dimension of the underlying feature space \cite{gyorfi2006distribution,tsybakov2008introduction}: these bounds describe best the observed behavior in many modern linear regressions, where the data are inherently high-dimensional or where a kernel associated to a high-dimensional feature map is used. In this paper, we provide nonparametric bounds for stochastic gradient descent under the noiseless linear model and under small perturbations of this model.  

Under the noiseless linear model, we assume that there exists a ground-truth linear relation $Y = \langle \theta_*, X \rangle$ between the feature vector $X$ and the output $Y \in \R$. The feature vector $X$ may be itself a non-linear transformation of the inputs $U$, explicitly computed through a feature map $X = \Phi(U)$ or implicitly defined through a positive-definite kernel $k(U,U')$ \cite{hofmann2008kernel}. The noiseless linear model assumes that there exists a linear predictor in feature space with zero generalization error. The difficulty to approximate this optimal prediction rule $\theta_*$ from independent identically distributed (i.i.d.) samples $(X_1,Y_1), \dots, (X_n, Y_n)$ depends on some measure of the complexity of $\theta_*$. 

The noiseless assumption is relevant for some basic vision or sound recognition tasks, where there is no ambiguity of the output $Y$ given the input $U$, but the rule determining the output from the input can be complex. An example from \cite[Section 6]{jun2019kernel} is the classification of images of cats versus dogs. For typical images, the output is unambiguous; humans indeed achieve a near-zero error. In sound recognition, one could think of the recovery of the melody from a tune, an unambiguous (but tremendously complex!) task.

 Note that in the noiseless model, there is still the randomness of the sampling of $X_1, \dots, X_n$, sometimes called multiplicative noise because algorithms end up multiplying random matrices \cite{dieuleveut2017harder}. Given those inputs, the outputs $Y_1, \dots, Y_n$ are deterministic: there is no additive noise, and thus the noiseless linear model we consider in this paper is a simplification of problems with low additive noise.

The large dimension and number of samples in modern datasets motivate the use of first-order online methods \cite{bottou2005line,bottou2008tradeoffs}. We study the archetype of these methods: single-pass, constant step-size stochastic gradient descent with no regularization, referred to as simply ``SGD'' in the following. 

\textbf{Contributions.} Our theoretical results and simulation agree to the following: under the noiseless linear model, the iterates of SGD converge to the optimum $\theta_*$ and the generalization error of SGD vanishes as the number of samples increases. Moreover, the convergence rate of SGD is determined by the minimum of two parameters: the regularity of the optimum $\theta_*$ and the regularity of the feature vectors $X$, where regularities are measured in terms of power norms of the covariance matrix $\Sigma = \E[X \otimes X]$, see Section \ref{sec:general-theory} for precise definitions and statements. Our analysis of the convergence is tight as we prove upper and lower bounds on the performance of SGD that almost match. Thus SGD shows some adaptivity to the complexity of the problem. In Section \ref{sec:robustness}, we study the robustness of our results when the noiseless linear assumption does not hold, but the generalization error of the optimal linear regression is small. We prove that the asymptotic generalization error of SGD deteriorates by a constant factor, proportional to this optimal generalization error. 

Two extensions of our results are studied. First, in Section \ref{sec:kernel}, the extension to kernel regression is derived, with, as a special case, the application to the interpolation of a real function $f_*$ on the torus $[0,1]^d$ from the observation of its value at randomly uniformly sampled points. In the latter case, we show that the rate of convergence depends on the Sobolev smoothness of the function $f_*$ and of the interpolating kernel. Second, beyond supervised learning, our abstract result can be seen as a result on products of i.i.d.~linear operators on a Hilbert space. In Section \ref{sec:averaging}, we use this result to study a linear stochastic process, the averaging process on a graph, which models a key algorithmic step in decentralized optimization, the gossip algorithm~\cite{shah2009gossip,nedic2010constrained}. We prove polynomial convergence rates depending on the spectral dimension of the graph. Finally, in Section \ref{sec:gaussian}, a toy application instantiates our results in the special case of Gaussian features.

\textbf{Comparison to the existing literature on linear / kernel regression.} There is an extensive research on the performance of different estimators in nonparametric supervised learning, however almost all of them do not consider the special case of the noiseless linear model \cite{gyorfi2006distribution,caponnetto2007optimal,tsybakov2008introduction,fischer2017sobolev}. The difference is significant; for instance, rates faster than $O(n^{-1})$ for the least-square risk are impossible with additive noise, while in this paper we prove that SGD can converge with arbitrarily fast polynomial rates. Some of these works analyse the performance of SGD \cite{ying2008online,bach2013non,tarres2014online,rosasco2015learning,dieuleveut2016nonparametric,dieuleveut2017harder,lin2018optimal,pillaud2018statistical,mucke2019beating}. However, because of the additive noise of the data, convergence requires averaging or decaying step sizes. As a notable exception,  \cite{jun2019kernel} studies a variant of kernel regularized least-squares and notices that the rate of convergence improves on noiseless data compared to noisy data.
 However, their rates are not directly comparable to ours as they assume that the optimal predictor is outside of the kernel space while we focus in Section \ref{sec:kernel} on the attainable case where the optimal predictor is in this space. We make a more precise comparison of this work with our results in Remark \ref{rmk:related-literature}. 

While our work focuses on the test error, a recent trend studies the ability of SGD to reach zero training error in the so-called ``interpolation regime'', that is in over-parametrized models where a perfect fit on the training data is possible \cite{schmidt2013fast,ma2017power,cevher2019linear}. Even with a fixed step size, SGD is shown to achieve zero-training error. However, these results are significantly different from ours: zero training error does not give any information on the generalization ability of the learned models, and the ``interpolation regime'' does not imply the noiseless model. The authors of \cite{vaswani2018fast} study a mixed framework that includes both the interpolation regime and the noiseless model, depending on whether SGD is seen as a stochastic algorithm minimizing the generalization error or the training error. An acceleration of SGD is studied, depending on the convexity property of the loss, but not on the nonparametric regularity of the problem. 

The field of scattered data approximation \cite{wendland2004scattered} studies the estimation of a function from the observation of its values at (possibly random) points, considered in Section \ref{sec:kernel}. Again, most of the work focuses on the case where the observation of the values is noisy. We found two exceptions that consider the noiseless case. In \cite{bauer2017nonparametric}, a minimax rate of $\Omega((\log n / n)^{p/d})$ is shown for estimating a $p$-smooth function on $[0,1]^d$ in $L^\infty$ norm using $n$ independent uniformly distributed points; the minimax rate is reached with a spline estimate. In \cite{kohler2013optimal}, a minimax rate of $\Omega(1/n^{p})$ in shown for the same problem, but in the special case of $d=1$ and estimation in $L^1$ norm; the minimax rate is reached with some nearest neighbor polynomial interpolation. Our results are not rigorously comparable with these as we consider the approximation in $L^2$ norm and a definition of smoothness different from theirs. However, roughly speaking, our convergence rate in Section \ref{sec:kernel} of $\Omega(1/n^{1-d/(2p+d)})$ when $p>d/2$ is much slower than theirs. Note that previous estimators could not be computed in an online fashion, and thus have a significantly larger running time. But in general, this suggests that SGD might not achieve the nonparametric minimax rates under the noiseless linear model.

\section{Linear regression}
\label{sec:general-theory}

\subsection{Setting and main results}
\label{sec:main-result}

We consider the regression problem of learning the linear relationship between a random feature variable $X \in \cH$ and a random output variable $Y \in \R$. The feature space $\cH$ is assumed to be a Hilbert space with scalar product $\left\langle . , . \right \rangle$ and norm $\Vert . \Vert$. We assume a noiseless linear model: there exists $\theta_* \in \cH$ such that $Y = \left\langle \theta_*, X \right\rangle$ almost surely (a.s.). In the online regression setting, we learn $\theta_*$ from i.i.d.~observations $(X_1,Y_1), (X_2,Y_2), \dots$ of $(X,Y)$. SGD proceeds as follows: it starts with the non-informative initialization $\theta_0 = 0$ and at iteration $n$, with current estimate $\theta_{n-1}$, it estimates the risk function on the observation $(X_n,Y_n)$, 
$\cR_n(\theta) =  \left(\left\langle \theta, X_n \right\rangle - Y_n\right)^2/2$
and it performs one step of gradient descent on $\cR_n$:
\begin{align}
\theta_n &= \theta_{n-1} - \gamma \nabla \cR_n(\theta_{n-1}) \nonumber \\
&= \theta_{n-1} - \gamma\left( \left\langle \theta_{n-1}, X_n\right\rangle - Y_n\right) X_n  \nonumber\\
&= \theta_{n-1} - \gamma \left\langle \theta_{n-1}-\theta_*, X_n\right\rangle X_n \label{eq:SGD_iteration} \, .
\end{align}
The risk $\cR_n(\theta)$ is an unbiased estimate of the population risk, also called generalization error,
\begin{align*}
\cR(\theta) &= \frac{1}{2} \E\left[\left(\left\langle \theta,X\right\rangle-Y\right)^2\right]=  \frac{1}{2} \E\left[\left\langle \theta-\theta_*,X\right\rangle^2\right] \, .
\end{align*}
We assume the feature variable to be uniformly bounded, namely that there exists a constant $R_0 < \infty$ such that 
\begin{equation}
 \Vert X \Vert^2 \leq R_0 \qquad \text{a.s.} \label{eq:def-R0-strong}
\end{equation}
We can then define the covariance operator $\Sigma = \E\left[X \otimes X \right]$ of $X$, where if $x \in \cH$, $x \otimes x$ is the bounded linear operator $\theta \in \cH \mapsto \langle \theta, x \rangle x$. Finally, note that, 
\begin{equation*}
\cR(\theta) = \frac{1}{2}\left\langle \theta - \theta_* , \Sigma\left(\theta - \theta_*\right)\right\rangle \, .
\end{equation*}
We do not assume that the linear operator $\Sigma$ is inversible as this is incompatible in infinite dimension with the boundedness assumption in Eq.~\eqref{eq:def-R0-strong}. Throughout this paper, we use the following convenient notation: if $\alpha$ is a positive real and $\theta$ a vector, $
\left\Vert \Sigma^{-\alpha/2}\theta \right\Vert^2 = \left\langle \theta, \Sigma^{-\alpha}\theta \right\rangle := \inf \left\{ \Vert \theta' \Vert^2 \, \middle\vert \, \theta' \text{ such that }\theta = \Sigma^{\alpha/2}\theta'  \right\}$,
with the convention that it is equal to $\infty$ when $\theta \notin \Sigma^{\alpha/2}(\cH)$. We have two theorems (upper and lower bounds) showing tight convergence rates for SGD.

\begin{thm}[upper bound]
	\label{thm:main-result-upper-bound}
	Assume that there exists a non-negative real number $\alphalow$ such that 
	\begin{enumerate}[label = (\alph*), noitemsep, nolistsep]
		\item\label{ass:reg-opt}(regularity of the optimum) $\theta_* \in \Sigma^{\alphalow/2}(\cH)$, i.e., $\Vert \Sigma^{-\alphalow/2} \theta_*\Vert < \infty$, and
		\item\label{ass:reg-feature}(regularity of the feature vector) $X \in \Sigma^{\alphalow/2}(\cH)$ a.s., and there exists a constant $R_\alphalow < \infty$ such that $\Vert \Sigma^{-\alphalow/2} X\Vert^2 \leq R_\alphalow$ a.s. 
	\end{enumerate}
	Assume further $0 <\gamma \leq 1/R_0$. The iterates $\theta_n$ of SGD with step-size $\gamma$ satisfy for all $n \geq 1$, 
	\begin{enumerate}[noitemsep, nolistsep]
		\item (reconstruction error) \hspace{1.8cm}
$\displaystyle
		\E\left[\Vert \theta_n - \theta_*\Vert^2 \right] \leq \frac{C}{n^\alphalow} \, , 
		$
		
		\vspace{0.1cm}
		
		\vspace*{.2cm}
		
		\item (generalization error) \hspace{1.5cm}
		$\displaystyle \min_{k=0, \dots, n } \E\left[\cR(\theta_k)\right] \leq \frac{C'}{n^{\alphalow + 1}} \, ,
		$
	\end{enumerate}
	where $\displaystyle C =  \frac{\alphalow^{\alphalow}}{\gamma^{\alphalow}} \left(\Vert \Sigma^{-\alphalow/2} \theta_* \Vert^2 + \frac{R_\alphalow}{R_0} \Vert \theta_* \Vert^2 \right)$ and $\displaystyle C' = 2^{\alphalow} \frac{\alphalow^\alphalow}{\gamma^{\alphalow+1}} \left(\Vert \Sigma^{-\alphalow/2} \theta_* \Vert^2 + \frac{R_\alphalow}{R_0} \Vert \theta_* \Vert^2 \right)$.
\end{thm}

Assumption \ref{ass:reg-opt} is classical in the non-parametric kernel literature \cite{caponnetto2007optimal}: it is often called \emph{complexity of the optimum}, or \emph{source condition}. Assumption \ref{ass:reg-feature} is made in \cite{pillaud2018statistical}. It implies that 
\begin{equation*}
\Tr(\Sigma^{1-\alphalow}) = \E[\Tr(XX^T\Sigma^{-\alphalow})] = \E[X^T\Sigma^{-\alphalow}X] \leq R_{\alphalow} \, .
\end{equation*}This last condition, called \emph{capacity condition} \cite{pillaud2018statistical}, is sometimes stated under the form of a given decay of the eigenvalues of $\Sigma$; it is related to the \emph{effective dimension} of the problem \cite{caponnetto2007optimal}. 

\begin{thm}[lower bound]
	\label{thm:main-result-lower-bound}
	 Assume that there exists a positive real number $\alphahigh$ such that one of the two following conditions holds:
	\begin{enumerate}[label = (\alph*), noitemsep, nolistsep]
		\item(irregularity of the optimum) $\theta_* \notin \Sigma^{\alphahigh/2}(\cH)$, i.e., $\Vert \Sigma^{-\alphahigh/2} \theta_*\Vert = \infty$, or
		\item(irregularity of the feature vector) with positive probability, $X \notin \Sigma^{\alphahigh/2}(\cH)$ and $\langle X, \theta_* \rangle \neq 0$. 
	\end{enumerate}
	Assume further $0 <\gamma \leq 1/R_0$. The iterates $\theta_n$ of SGD with step-size $\gamma$ satisfy for all $\varepsilon > 0$, 
	\begin{enumerate}[noitemsep, nolistsep]
		\item (reconstruction error)  $\E\left[\Vert \theta_n - \theta_*\Vert^2 \right]$ is not asymptotically dominated by $1/n^{\alphahigh+\varepsilon}$, 
		\item (generalization error) $\E\left[\cR(\theta_n)\right]$ is not asymptotically dominated by $1/n^{\alphahigh+1+\varepsilon}$.
	\end{enumerate}
\end{thm}

The take-home message of Theorems \ref{thm:main-result-upper-bound}, \ref{thm:main-result-lower-bound} is that the convergence rate of SGD is governed by two real numbers: the regularity $\alpha_1$ of the optimum, that is the supremum of all $\alphalow$ such that $\theta_* \in \Sigma^{\alphalow/2}(\cH)$, and the regularity $\alpha_2$ of the features, that is the supremum of all $\alphalow$ such that $X \in \Sigma^{\alphalow/2}(\cH)$ almost surely. The polynomial convergence rate of SGD is roughly of the order of $n^{-\alpha}$ for the reconstruction error and $n^{-\alpha-1}$ for the generalization error with $\alpha = \min(\alpha_1, \alpha_2)$: one of the two regularities is a bottleneck for fast convergence. See Section \ref{sec:kernel} for an application to the optimal choice of a reproducing kernel Hilbert space. The exponent $\alpha_1$ corresponds to the decay of the errors of the gradient descent on the population risk $\cR$. However, due to the multiplicative noise, the convergence of SGD is slowed down by the irregularity of the feature vectors if $\alpha_2 < \alpha_1$. 

In the theorems, the constraint on the step-size $0<\gamma \leq 1/R_0$ is independent of the time horizon $n$ and of the regularities $\alpha_1, \alpha_2$. Thus fixed step-size SGD shows some adaptivity to the regularity of the problem. 

In Section \ref{sec:applications}, we give extensive numerical evidence that the polynomial rates $n^{-\alpha}$ and $n^{-(\alpha+1)}$ in the bounds are indeed sharp in describing convergence rate of SGD. 

We end this section with a few remarks on Theorems \ref{thm:main-result-upper-bound}, \ref{thm:main-result-lower-bound}. They articulate the significance of the results, but are non-essential to the rest of this paper. 

\begin{remark}
	\label{rmk:pointwise-decay-risk}
	Our upper bound and lower bound on the generalization errors do not match exactly. Indeed, we prove an upper bound on the \emph{minimum} risk of the past iterates, where we prove a lower bound on a larger quantity, the risk of the \emph{last} iterate. To the best of our knowledge, it is an open question whether one can prove an upper bound for the last iterate under our assumptions: more precisely, does $\E\left[\cR(\theta_n) \right] \leq C''/n^{\alphalow + 1}$ hold for some constant~$C''$?
\end{remark}

\begin{remark}[related literature]
	\label{rmk:related-literature}
	In the case $\alphalow = 0$, where no regularity assumption is made on the optimum or the features (apart from being bounded), we upper-bound $\min_{k=1, \dots, n } \E\left[\cR(\theta_k)\right]$ by $O(n^{-1})$. A similar result was shown in~\cite{bach2013non}: the excess risk for averaged constant-step size SGD is asymptotically dominated by $n^{-1}$ on any least-squares problem--not necessarily a noiseless one. It is remarkable that under the noiseless linear setting, no averaging or decay of the step-size is needed to obtain the same convergence rate. 
	
	The article \cite{jun2019kernel} also studies the performance of an algorithm, a variant of kernel regularized least-squares, in the noiseless non-parametric setting. However, they do not exploit when the function is more regular than being in the kernel space, i.e., when $\alpha_1 > 0$ with our notation, $\beta > 1/2$ with theirs. In fact, they leave this case as an open problem in their Section 6. Thus, a fair comparison can only be made when $\alpha_1 = 0, \beta = 1/2$. In this case, SGD and the algorithm of \cite{jun2019kernel} both achieve the same rate $O(n^{-1})$. 
\end{remark}

\begin{remark}
	\label{rmk:weak-assumptions}
	The theorems stated above stay true if one weakens the assumptions in the following way, where $\preccurlyeq$ denotes the semi-definite order:
	\begin{itemize}[noitemsep, nolistsep]
		\item assume $\E\left[\Vert X \Vert^2 X \otimes X\right] \preccurlyeq R_0 \Sigma$ instead of $\Vert X \Vert^2 \leq R_0$ a.s., and
		\item assume $\E\left[\left\langle X , \Sigma^{-\alphalow} X\right\rangle  X \otimes X\right] \preccurlyeq R_\alphalow \Sigma$ instead of $\left\langle X , \Sigma^{-\alphalow}X \right\rangle \leq R_\alphalow$ a.s.
	\end{itemize}
	This weaker set of assumptions is useful in the case of non-bounded features, like the Gaussian features of Appendix \ref{sec:gaussian}. We thus take special care in using only these weaker assumptions in the proofs of Theorems \ref{thm:main-result-upper-bound}, \ref{thm:main-result-lower-bound}, \ref{thm:general-result-upper-bound} and \ref{thm:general-result-lower-bound}. However we prefer stating results with the stronger assumptions for the sake of clarity. 
\end{remark}

\begin{remark}[Application of Theorem \ref{thm:main-result-upper-bound} in finite dimension]
	If $\cH$ is finite-dimensional and $\Sigma$ is of full rank, the assumptions of Theorem \ref{thm:main-result-upper-bound} hold for any $\alphalow \geq 0$. Thus SGD converges faster than any polynomial; in fact one can check that an exponential upper bounds on the reconstruction and generalization errors of the form $C''\exp(-\lambda_{\min}(\Sigma) t)$ hold, where $\lambda_{\min}(\Sigma)$ is the smallest eigenvalue of $\Sigma$. Although the latter bound is asymptotically better than polynomial rates, for moderate time scales the polynomial rates may describe best the observed behavior; for an illustration of this fact on the averaging process, see Section \ref{sec:averaging} and in particular the discussion following Corollary \ref{coro:averaging}. 
\end{remark}

Theorems \ref{thm:main-result-upper-bound} and \ref{thm:main-result-lower-bound} are extended in the next section and proved in Appendices \ref{sec:proof-thm-general-upper} and \ref{sec:proof-thm-general-lower} respectively. The generalization of Theorem \ref{thm:main-result-upper-bound} beyond the noiseless linear model is exposed in Section \ref{sec:robustness}. The reader interested mostly by applications of Theorems \ref{thm:main-result-upper-bound} and \ref{thm:main-result-lower-bound} can jump directly to Section \ref{sec:applications}.

\subsection{Regularity functions and general results} 

The main difficulty in the proof of Theorems \ref{thm:main-result-upper-bound} and \ref{thm:main-result-lower-bound} is that deriving closed recurrence relations for the expected reconstruction and generalization errors is not straightforward. In this paper, we propose to study the norm of $\theta_n- \theta_*$ associated to different powers of the covariance $\Sigma$. More precisely, define
\begin{equation}
\varphi_n(\beta) = \E\left[\left\langle \theta_n - \theta_*, \Sigma^{-\beta}\left(\theta_n - \theta_*\right)\right\rangle\right] \in [0,\infty] \, , \qquad \beta \in \R \, . \label{eq:regularity-function}
\end{equation}
We call $\varphi_n$ the regularity function at iteration $n$. In particular, 
\begin{align*}
&\varphi_n(0) = \E[\Vert \theta_n - \theta_* \Vert^2] &&\text{and} &&\varphi_n(-1) =  2 \E[\cR(\theta_n)] \, .
\end{align*}
The sequence of regularity functions $\varphi_n$, $n\geq 1$ satisfies a closed recurrence inequality (Property \ref{prop:regularity-ineq} in Appendix \ref{sec:proof-thm-general-upper}) which is central to our proof strategy. Theorems \ref{thm:main-result-upper-bound} and \ref{thm:main-result-lower-bound} can be extended to the following estimates on the regularity functions $\varphi_n(\beta)$ on the full interval $\beta \in [-1,\alphahigh]$ (see proofs in Appendices \ref{sec:proof-thm-general-upper} and \ref{sec:proof-thm-general-lower} respectively). .

\begin{thm}[upper bound]
	\label{thm:general-result-upper-bound}
	Under the assumptions of Theorem \ref{thm:main-result-upper-bound}, we have for all $n \geq 1$,
	\begin{enumerate}[noitemsep, nolistsep]
		\item\label{concl:positive} for all $\beta \in [0,\alphalow]$, \hspace{2.9cm}
		$\displaystyle \varphi_n(\beta) \leq \frac{C}{n^{\alphalow-\beta}} \, ,$ \hspace{2cm}
		\item\label{concl:negative} for all $\beta \in [-1,0)$, \hspace{1.5cm}
		$\displaystyle \min_{k=0, \dots, n } \varphi_k(\beta) \leq \frac{C'}{n^{\alphalow-\beta}} \, , $\hspace{1.5cm}
	\end{enumerate}
		where $\displaystyle C = \frac{\alphalow^{\alphalow-\beta}}{\gamma^{\alphalow-\beta}} \left(\Vert \Sigma^{-\alphalow/2} \theta_* \Vert^2 + \frac{R_\alphalow}{R_0} \Vert \theta_* \Vert^2 \right)$, $\displaystyle C' = 2^{\alphalow - \beta} \frac{\alphalow^{\alphalow}}{\gamma^{\alphalow-\beta}} \left(\Vert \Sigma^{-\alphalow/2} \theta_* \Vert^2 + \frac{R_\alphalow}{R_0} \Vert \theta_* \Vert^2 \right)$.
\end{thm}
\begin{thm}[lower bound]
	\label{thm:general-result-lower-bound}
	Under the assumptions of Theorem \ref{thm:main-result-lower-bound}, for all $\beta \in [-1,\alphahigh]$, for all $\varepsilon > 0$, $\varphi_n(\beta)$ is not asymptotically dominated by $1/n^{\alphahigh-\beta+\varepsilon}$. 
\end{thm}

\section{Applications}
\label{sec:applications}



\subsection{Kernel methods and interpolation in Sobolev spaces}
\label{sec:kernel}

A main case of application of our results is the reproducing kernel Hilbert space (RKHS) setting \cite{hofmann2008kernel}. In this setting, the space $\cH$ is typically large or infinite-dimensional, and we do not have a direct access to the feature variable $X \in \cH$. Instead, we have access to some random input variable $U \in \cU$ such that $X = \Phi(U)$ for some fixed feature map $\Phi:\cU \to \cH$. It is then natural to associate a vector $\theta \in \cH$ with the function $f_\theta \in L^2(\cU)$ defined by
\begin{equation*}
f_\theta(u) = \left\langle \theta, \Phi(u) \right\rangle \, .
\end{equation*}
If the positive-definite kernel $k(u,u') = \langle \Phi(u), \Phi(u') \rangle$ can be computed efficiently, SGD can be ``kernelized'' \cite{ying2008online,tarres2014online,rosasco2015learning,dieuleveut2016nonparametric}, i.e., the iteration can be written directly in terms of $f_n := f_{\theta_n}$:
\begin{align*}
f_n &= f_{n-1} - \gamma(f_{n-1}(U_n)-Y_n)k(U_n, .) \\
&= f_{n-1} - \gamma(f_{n-1}-f_*)(U_n)k(U_n, .)
\end{align*}
where $X_n = \Phi(U_n)$ and $f_*(u) := f_{\theta_*}(u) =  \langle \theta_*, \Phi(u) \rangle$. Note that in the kernel literature, the mapping $\theta \mapsto f_\theta$ is used to identify $\cH$ with a subspace of $L^2(\cU)$; indeed, if $\Sigma = \E\left[\Phi(U) \otimes \Phi(U)\right]$ has dense range, the mapping is injective. Using this identification, Theorems \ref{thm:main-result-upper-bound} and \ref{thm:main-result-lower-bound} can be applied to obtain bounds in the ``attainable'' case, meaning that the optimal predictor $f_* \in L^2(\cU)$ is in the RKHS $\cH$. This gives decay rates for the RKHS norm $\Vert f_n - f_* \Vert := \Vert \theta_n - \theta_* \Vert$ which is inherited from $\cH$, but also for the population risk $\cR(\theta_n)$ which is reinterpreted as the half squared $L^2$-distance between the associated $f_n$ and the optimal predictor $f_*$. Indeed,
\begin{equation*}
\cR(\theta_n) = \frac{1}{2}\E\left[\left\langle \theta_n - \theta_*, \Phi(U) \right\rangle^2\right] = \frac{1}{2} \E\left[\left(f_n(U)-f_*(U)\right)^2\right] = \frac{1}{2} \Vert f_n - f \Vert_{L^2(U)}^2 \, .
\end{equation*} 

\textbf{Application: interpolation in Sobolev spaces.} To illustrate our results, we consider the case where $\cU$ is the torus $[0,1]^d$, $U$ is uniformly distributed on $\cU$ and $k$ is a translation-invariant kernel: $k(u,u') = t(u-u')$ where $t$ is a square-integrable $1$-periodic function on $[0,1]^d$. The kernel $k$ is positive-definite if and only if the Fourier transform of $t$ is positive \cite{wahba1990spline}. This imposes, in particular, that $t$ is maximal at $0$. Thus the update rule 
\begin{align}
\label{eq:iteration-interpolation-function}
f_n &= f_{n-1} - \gamma\left(f_{n-1}(U_n)-f_*(U_n)\right)t(.-U_n) 
\end{align}
corrects $f_n$ so that the value $f_n(U_n)$ is closer to the observed value $f_*(U_n)$ than $f_{n-1}(U_n)$. Points near $U_n$ are also updated in the same direction, thus the algorithm should converge rapidly if the function $f_*$ is smooth. Our work derives the polynomial convergence rate as a function of the smoothness of $f_*$ and $t$. The smoothness of functions is measured with the Sobolev spaces $H^s_{\text{per}}$. A function $f$ with Fourier serie $\hat{f}$ belongs to $H^s_{\text{per}}$ if 
\begin{equation*}
\Vert f \Vert^2_{H^s_{\text{per}}} = \sum_{k\in \Z^d} \vert \hat{f}(k) \vert^2 \left(1 + |k|^2\right)^s < \infty \, .
\end{equation*}
Assume that the Fourier serie of $t$ satisfies a power-law decay: there exists $c, C > 0$ such that:
\begin{equation*}
c \left(1+|k|^2\right)^{-s/2-d/4} \leq \hat{t}(k) \leq C \left(1+|k|^2\right)^{-s/2-d/4} \, , \qquad k \in \Z^d \, .
\end{equation*}
This condition does not cover $C^\infty$ kernel, including the Gaussian kernel; it is relevant for less regular kernel, that have a power decay in Fourier. This condition is satisfied, for instance, by the Wendland functions \cite[Theorem 10.35]{wendland2004scattered}, or in dimension $d=1$ by the kernels corresponding to splines of order $s$, see \cite{wahba1990spline} or \cite{pillaud2018statistical}. The latter can be computed using the polylogarithm or--for special values of $s$--the Bernoulli polynomials. 

We have $t \in H^{s'}_{\text{per}}$ if and only if $s' < s$, thus $s$ measures the Sobolev smoothness of $k$. The operator $\Sigma$ is the convolution with $t$ and thus 
\begin{equation}
\label{eq:norm-rkhs-sobolev}
\Vert f \Vert^2 = \left\langle f , \Sigma^{-1}f \right\rangle_{L^2} \asymp \sum_{k\in\Z^d} \vert \hat{f}(k) \vert^2 \left( 1 + |k|^2 \right)^{s/2+d/4} = \Vert f \Vert^2_{H^{s/2+d/4}_{\text{per}}} \, ,
\end{equation}
where $\asymp$ denotes the equality up to positive multiplicative constants. To predict the convergence rate of \eqref{eq:iteration-interpolation-function}, we check the assumptions of Theorems \ref{thm:main-result-upper-bound}, \ref{thm:main-result-lower-bound}. Computations similar to \eqref{eq:norm-rkhs-sobolev} give 

\begin{enumerate}[label = (\alph*),noitemsep, nolistsep]
	\item (regularity of the optimum) 
	\begin{equation*}
	\left\langle f_* , \Sigma^{-\alpha}f_* \right\rangle \asymp \left\langle f_* , \Sigma^{-\alpha-1}f_* \right\rangle_{L^2} 
	\asymp \Vert f_* \Vert^2_{H^{(s/2+d/4)(\alpha+1)}_{\text{per}}}
	\end{equation*}
	Assume $f_* \in H^{r}_{\text{per}}$. We have $\left\langle f_* , \Sigma^{-\alpha}f_* \right\rangle < \infty $ if $ \alpha \leq \frac{2r}{s+d/2} - 1$.
	\item (regularity of the feature vector)  
	\begin{align*}
	\left\langle k(u,.) , \Sigma^{-\alpha}k(u,.) \right\rangle &= \Vert k(u,.) \Vert^2_{H^{(s/2+d/4)(\alpha+1)}_{\text{per}}} = \Vert t_s \Vert^2_{H^{(s/2+d/4)(\alpha+1)}_{\text{per}}} \\
	&= \sum_{k\in\Z^d} (1+k^2)^{(s/2+d/4)(\alpha-1)} \, .
	\end{align*}
	Thus $\left\langle k(u,.) , \Sigma^{-\alpha}k(u,.) \right\rangle < \infty $ if and only if $\alpha < 1 - \frac{d}{s+d/2}$.
\end{enumerate}
The regularities of the optimum and of the feature vector are non-negative if the smoothness $s$ of the kernel~$t$ satisfies $d/2<s\leq 2r-d/2$, where $r$ is the smoothness of $f_*$. In this case the polynomial rate of decay of the algorithm is given by the exponent
\begin{equation}
\label{eq:rate-sobolev}
\alpha_* = \min\left(  \frac{2r}{s+d/2} - 1 , 1 - \frac{d}{s+d/2}\right) \, .
\end{equation}
Note that, given a function $f_*$, this rate is maximal when $s=r$, i.e., the smoothness of the kernel coincides with the smoothness of the function, in which case $\alpha_* = 1-\frac{d}{r+d/2}$. Theorems \ref{thm:main-result-upper-bound}, \ref{thm:main-result-lower-bound} give the convergence rates in terms of $L^2$ norm and RKHS norm, which happens to be a Sobolev norm. The more general Theorems~\ref{thm:general-result-upper-bound} and~\ref{thm:general-result-lower-bound} gives convergence rates in terms of a continuity of fractional Sobolev norms, some weaker and some stronger than the RKHS norm.

\begin{figure}
	\includegraphics[width = 0.33\textwidth]{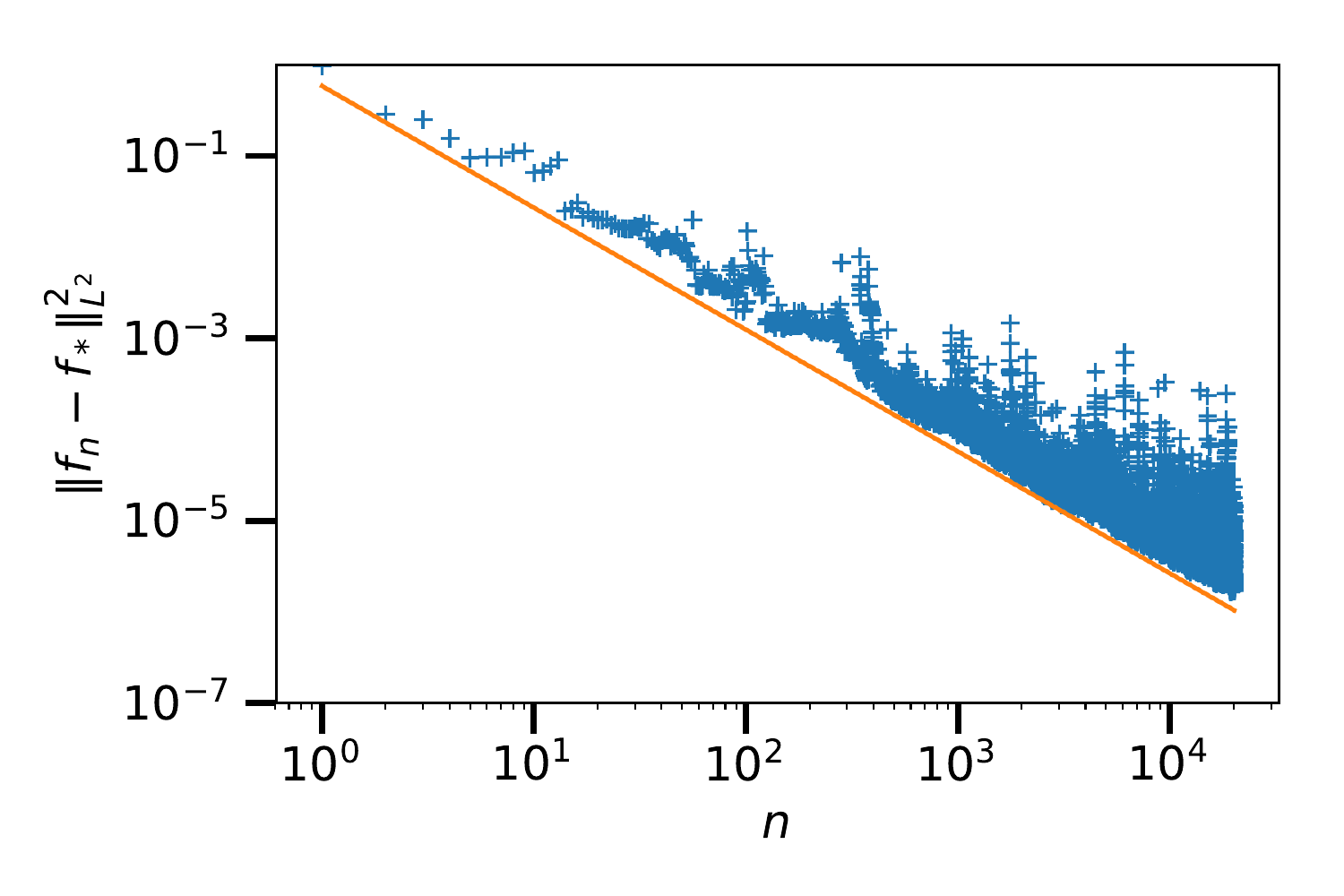}
	\includegraphics[width = 0.33\textwidth]{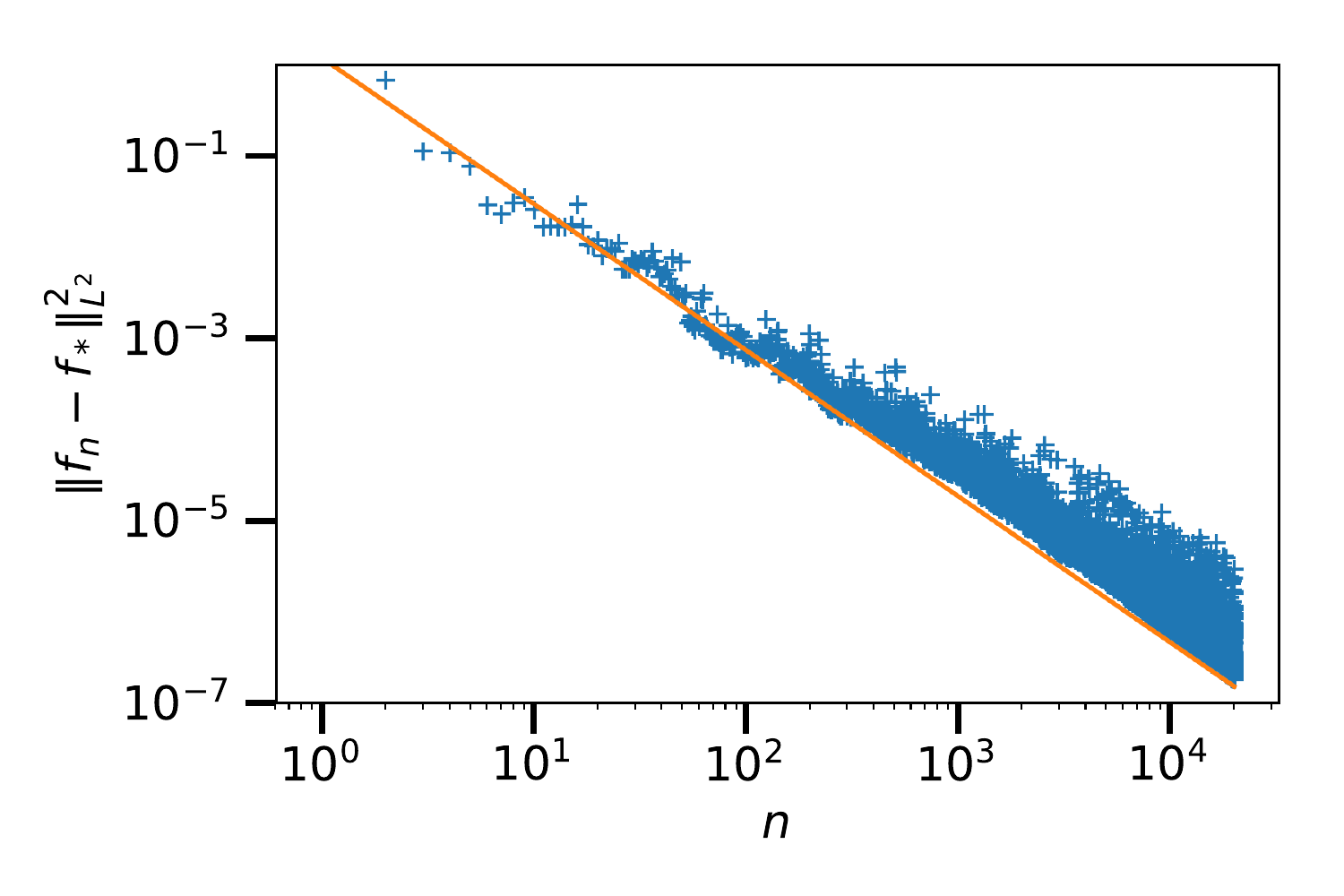}
	\includegraphics[width = 0.33\textwidth]{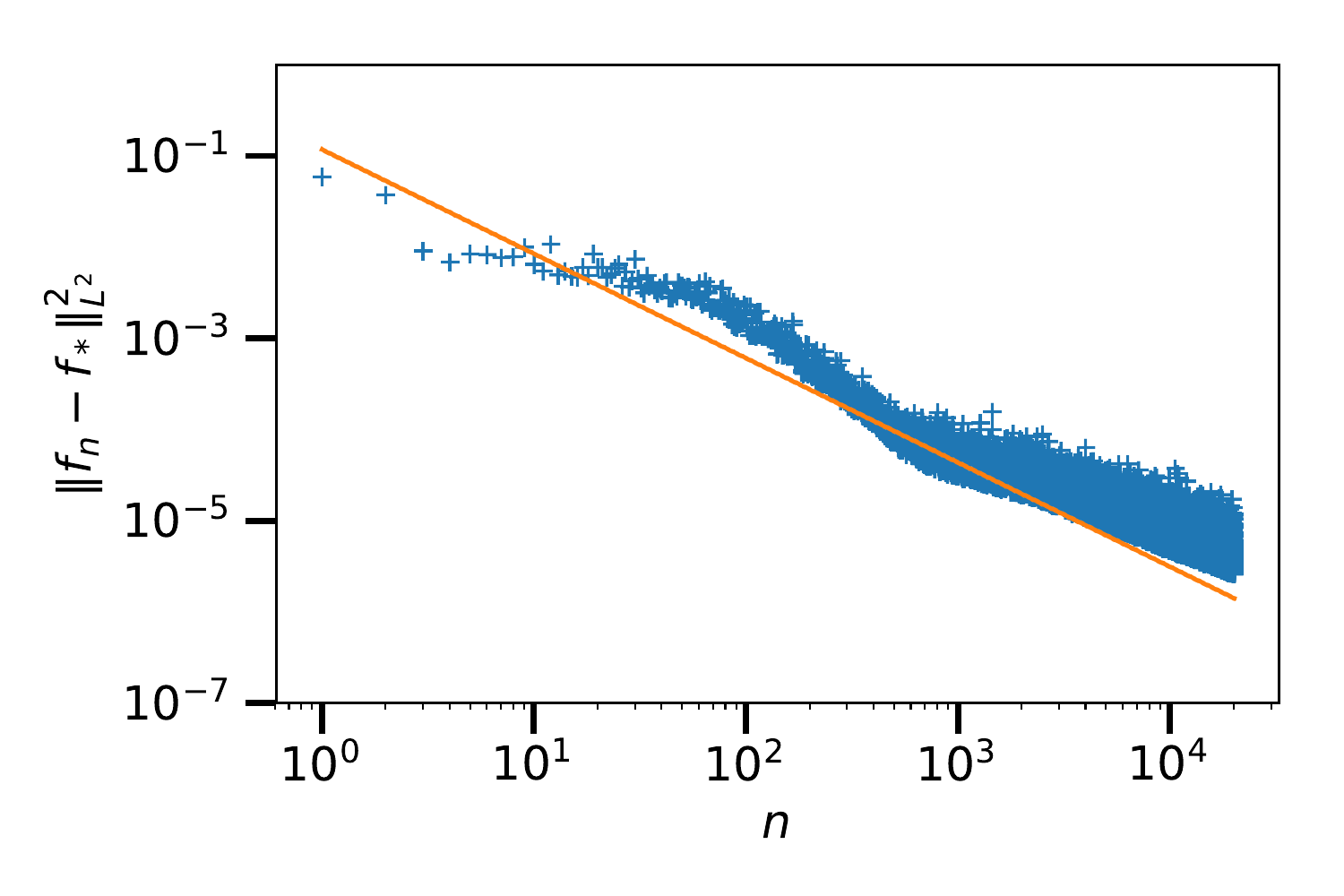}
	\vspace{-0.6cm}
	\caption{Interpolation of a function of smoothness $r=2$ using SGD with kernels of smoothness $s=1$ (left), $s=2$ (middle) and $s=3$ (right). Each plot represents one realization of the algorithm~\eqref{eq:iteration-interpolation-function}. The blue crosses represent the square $L^2$~norms $\Vert f_n - f_* \Vert^2_{L^2}$ as a function of the number of iterations $n$ and the orange lines represent the predicted polynomials rates $C/n^{\alpha_*+1}$, where $C$ is chosen to match best the empirical observations for each plot.}
	\label{fig:sobolev}
\end{figure}

In Figure \ref{fig:sobolev}, we show the decay of the $L^2$ norm in the interpolation of a function $f_*$ on $[0,1]$ of smoothness $2$ using kernels of smaller, matching and larger smoothness. In each case, the rate predicted by \eqref{eq:rate-sobolev} is sharp, and the convergence is indeed fastest when the smoothnesses match.

\subsection{Decay rate of the averaging process}
\label{sec:averaging}

The averaging process is a stochastic process on a graph, mostly studied as a model for asynchronous gossip algorithms on networks. Gossip algorithms are subroutines used to diffuse information throughout networks in distributed algorithms \cite{shah2009gossip}, in particular in distributed optimization~\cite{nedic2010constrained}.

Let $G$ be a finite undirected connected graph with vertex set $\cV$ of cardinality $N$ and edge set $\cE$ of cardinality~$M$. The averaging process is a discrete process on functions $x : \cV \to \R$ defined as follows. The initial configuration $x_0 = e_{v_\star}: \cV \to \R$ is the indicator function of some distinguished vertex $v_\star \in \cV$, i.e., $x_0(v_\star) = 1$ and $x_0(v) = 0$ if $v \neq v_\star$. At each iteration, we choose a random edge and replace the values at the ends of the edge by the average of the two current values. In equations, at iterations $n$, given $x_{n-1}$, sample an edge $e_n = \{v_n,w_n\}$ uniformly at random from $\cE$ and independently from the past, and define 
\begin{align}
&x_n(v_n) = x_n(w_n) = \frac{x_{n-1}(v_n) + x_{n-1}(w_n)}{2} \, , &&x_n(v) =x_{n-1}(v) \, , \quad v \neq v_n, w_n \, . \label{eq:updates-averaging}
\end{align}
As the graph is connected, all functions values $x_n(v), v \in \cV$ converge to $1/N$ as $n \to \infty$. The study of the averaging process aims at describing how the speed of convergence depends on the graph $G$. 

The averaging process can be seen as a prototype interacting particle system, or finite markov information-exchange process according to Aldous's terminology \cite{aldous2013interacting}. However, the linear structure of the updates of the averaging process makes the analysis simpler than in other interacting particle systems; this property is key in applying the results of Section \ref{sec:general-theory}.

In this section, we introduce a quantitive version of the notion of spectral dimension of a graph (see \cite{avrachenkov2019eigenvalues} and references therein for other definitions). We use this quantity to build polynomial convergence rates for the expected squared $\ell^2$-distance to optimum $\E\big[\sum_{v \in \cV} \left(x_n(v) - 1/N\right)^2\big]$ and for the expected energy $\E\big[\frac{1}{2}\sum_{\{v,w\}\in \cE} (x_n(v) - x_n(w))^2\big]$. The comparison with other known convergence bounds is made. We add numerical experiments showing that our bounds describe the observed behavior in some classical large graphs, for an intermediate number of iterations.

Let $L = \sum_{\{v,w\}\in \cE} (e_v-e_w)(e_v - e_w)^\top$ be the Laplacian of the graph. It is a positive semi-definite operator. The spectral measure of $L$ at a vertex $v \in \cV$ is the unique measure $\sigma_v$ such that for all continuous real function $f$,
\begin{equation*}
\langle e_v, f(L) e_v \rangle = \int \diff \sigma_v(\lambda) f(\lambda) \, .
\end{equation*}If $0 = \lambda_0 < \lambda_1 \leq  \dots \leq \lambda_{N-1}$ are the eigenvalues of $L$ and $u_0 = \bfone, u_1, \dots, u_{N-1}$ are the corresponding normalized eigenvectors, then 
\begin{equation*}
\sigma_v(\diff \lambda) = \sum_{i=0}^{N-1} (u_i(v))^2 \delta_{\lambda_i}(\diff\lambda) \, .
\end{equation*}
We say that $G$ is of spectral dimension $d \geq 0$ with constant $V>0$ if 
\begin{equation*}
\forall v \in \cV \, , \quad \forall E \in (0,\infty)\, , \quad  \sigma_v((0,E]) \leq V^{-1} E^{d/2} \, .
\end{equation*}
A typical example motivating this definition is the following.

\begin{proposition}
	\label{prop:torus}
	Let $\T^d_\Lambda$ denote the $d$-dimensional torus of side length $\Lambda$, i.e., the graph with vertex set $\cV = (\Z/\Lambda\Z)^d$ and edge set $\cE = \left\{\{v,w\} \, \middle\vert \,  v,w \in E, \Vert v-w \Vert_2 =1 \right\}$. The torus $\T^d_\Lambda$ is of spectral dimension $d$ with some constant $V(d)$ that depends on the dimension $d$ but not on the side length $\Lambda$. 
\end{proposition}

This result is proved in Appendix \ref{sec:proof-dim-torus}. Similar results were proved for supercritical percolation bonds in \cite{mathieu2004isoperimetry} and for the random geometric graphs in \cite{avrachenkov2019eigenvalues}.

When the graph is large, the probability of sampling a given edge decays to $0$. It is natural to define a rescaled time $t = n/M$ so that the expected number of times a given edge is sampled during a unit time interval does not depend on $M$ (and is equal to $1$). 

\begin{coro}[of Theorem \ref{thm:main-result-upper-bound}]
	\label{coro:averaging}
	Assume that $G$ is of spectral dimension $d$ with constant $V$, and denote $\delta_{\max}$ the maximal degree of the nodes in the graph.  Then, for all $t=n/M \geq 2$,  
	\begin{enumerate}[noitemsep, nolistsep]
		\item\label{concl:gossip-norm} \hspace{1.5cm}$\displaystyle
		\E\left[\sum_{v\in \cV}\left(x_{Mt}(v)-\frac{1}{N}\right)^2\right] \leq D(d,V,\delta_{\max}) \frac{\log t}{t^{d/2}} \, ,$
		\item\label{concl:gossip-energy} \hspace{0.2cm}
		$\displaystyle \min_{0\leq s\leq t} \E\left[\frac{1}{2} \sum_{\{v,w\}\in \cE} \left(x_{Ms}(v) - x_{Ms}(w)\right)^2\right] \leq D'(d,V,\delta_{\max}) \frac{\log t}{t^{d/2+1}} \, ,$
	\end{enumerate}
	where $\displaystyle D(d,V,\delta_{\max}) = \frac{2}{\log 2}d^{d/2+1}V^{-1}\delta_{\max}$ and $\displaystyle D'(d,V,\delta_{\max}) = \frac{2^{d/2+2}}{\log 2} d^{d/2+1} V^{-1} \delta_{\max}$.
\end{coro}
See Appendix \ref{sec:proof-averaging} for the proof. Note that as $G$ is a finite graph, $G$ can be of any spectral dimension $d$ for some potentially large constant $V$. However, for many families of graphs of increasing size, such as the toruses $\T^d_\Lambda$, $\Lambda \geq 1$, the spectral dimension constant $V$ corresponding to the dimension $d$ and the maximum degree $\delta_{\max}$ remain bounded independently of the size of the graph. In that case, the bounds of Corollary~\ref{coro:averaging} are independent of the size of the graph. 

\begin{figure}
	\includegraphics[width = 0.49\textwidth]{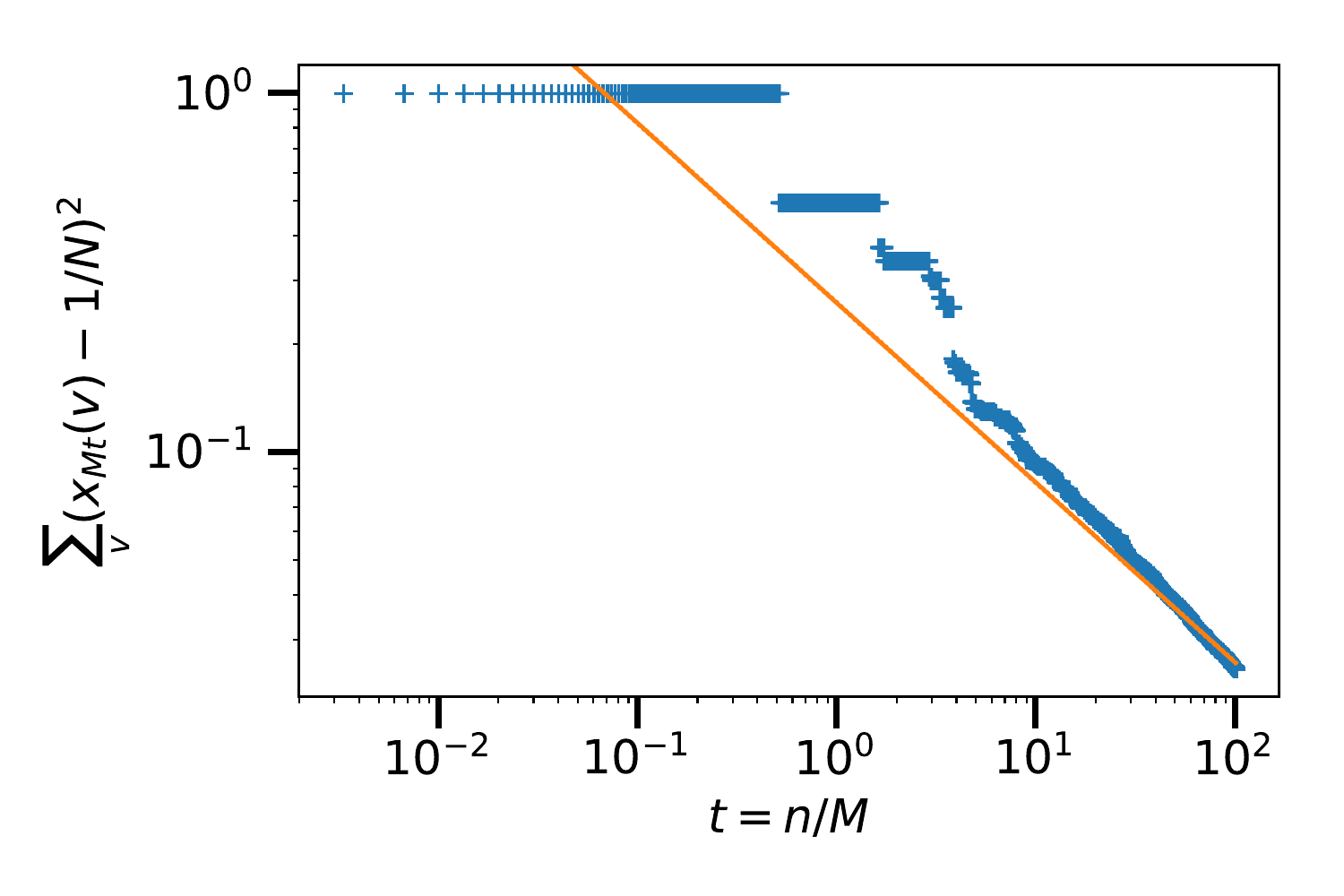}
	\includegraphics[width = 0.49\textwidth]{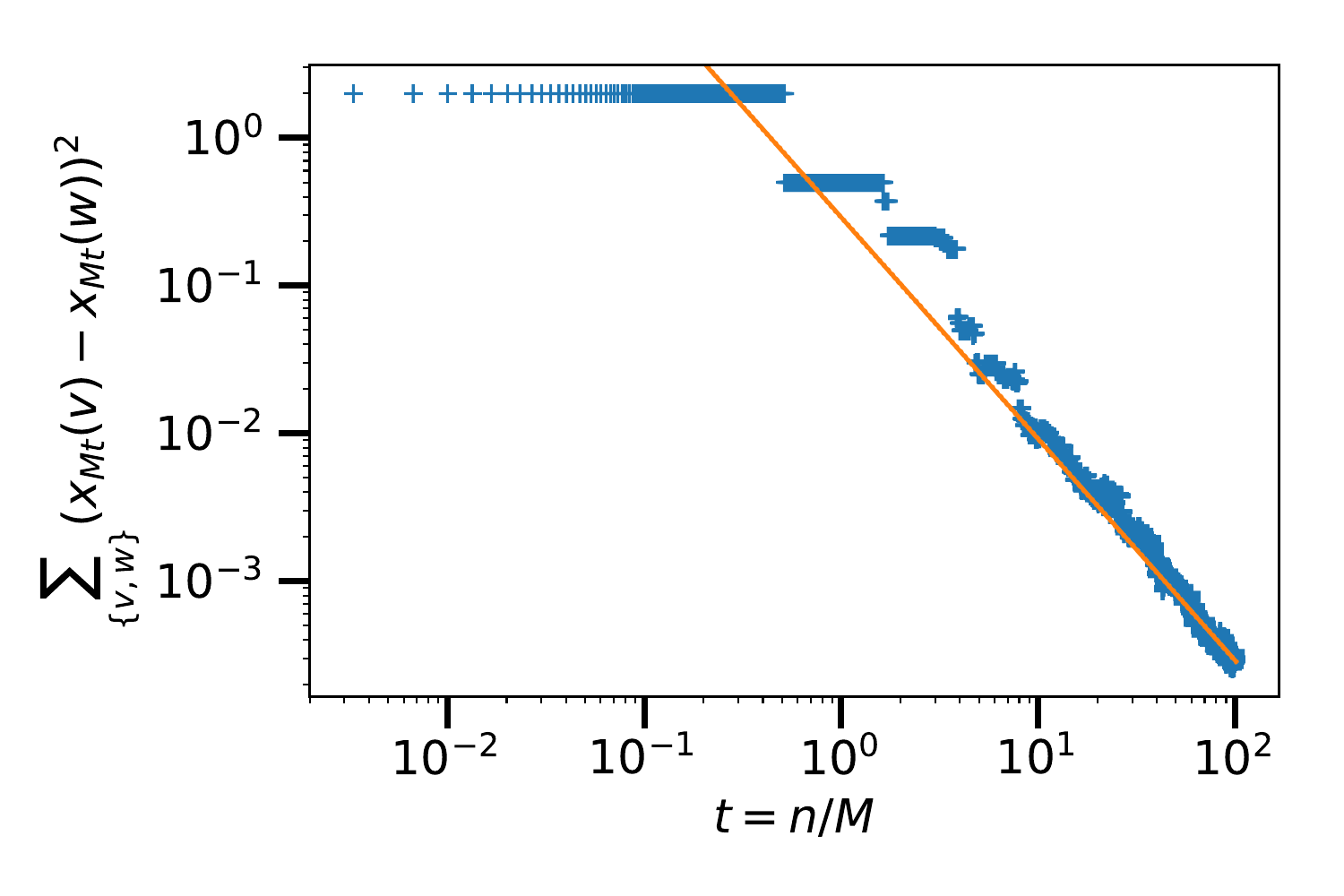}
	\vspace{-0.5cm}
	\includegraphics[width = 0.49\textwidth]{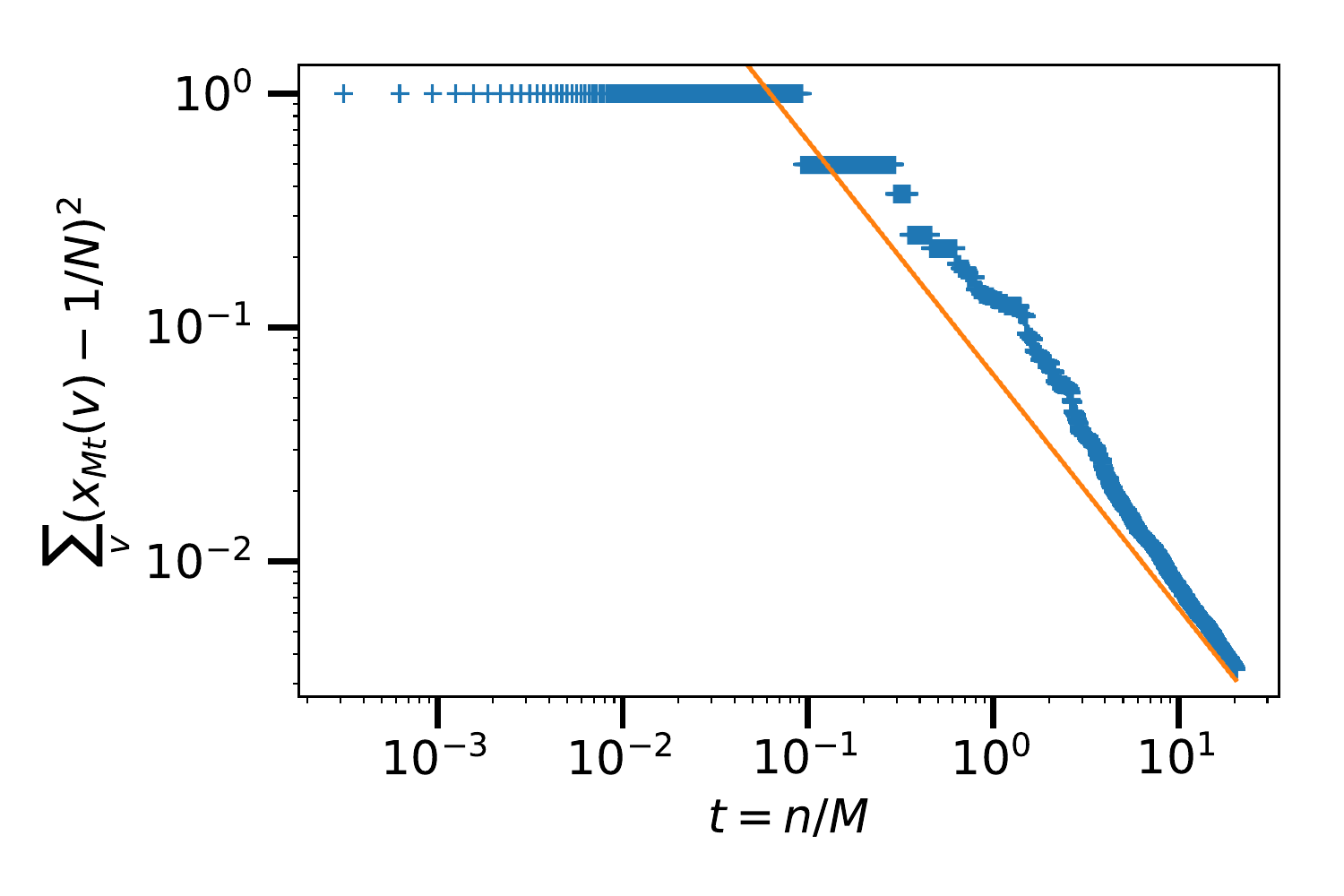}
	\includegraphics[width = 0.49\textwidth]{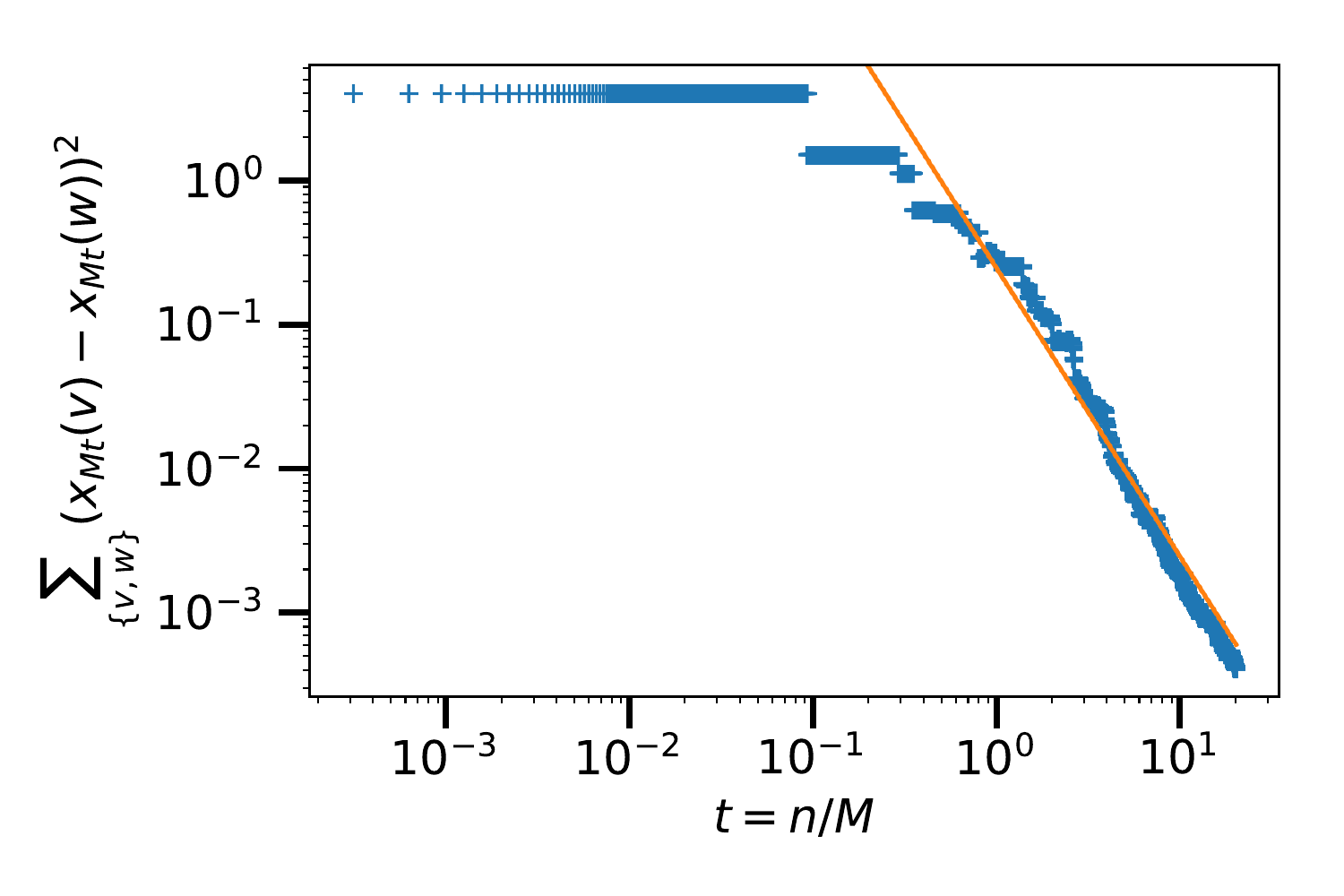}
	\caption{Convergence rates on the circle $\T^1_{300}$ (up) and on the two dimensional torus $\T^2_{40}$ (bottom). The convergence is measured in terms of squared $\ell^2$-distance to $\frac{1}{N}\bfone$ (left) and sum of the squared differences along the edges (right). In orange are the curves of the form $C/n^{d/2}$ and $C'/n^{d/2+1}$ where $C$ and $C'$ are constants chosen to match best the empirical observations for each plot.}
	\label{fig:simus-gossip}
\end{figure}

These bounds should be compared to the known exponential convergence bounds of \cite{aldous2012lecture} or \cite{shah2009gossip}: they are of the form $O(\exp(-\gamma t))$ where $\gamma$ is the spectral gap of the Laplacian of the graph, the distance between the two minimal eigenvalues of the Laplacian. Although asymptotically faster, these bounds are only relevant on the typical scale $t \gtrsim 1/\gamma$. In many graphs of interests, the spectral gap $\gamma$ vanishes as the size of the graph increases; for instance, when $G = \T^d_\Lambda$, $\gamma$ is of the order of $1/\Lambda^2$. As a consequence, for large graphs and moderate number of iterations, the spectral dimension based bounds describe the observed behavior where spectral gap based bounds do not apply. Indeed, in Figure \ref{fig:simus-gossip}, simulations on a large circle $\T^1_{300}$ and on a large torus $\T^2_{40}$ display polynomial decay rates, with polynomial exponents coinciding with those of the corresponding bounds of Corollary \ref{coro:averaging}. Note that, if pushed on a longer time scale, the simulations would have shown the exponential convergence due to finite graph effects. This incapacity of spectral gap to describe the transient behavior had already motivated the authors of \cite{berthier2020accelerated} to use the spectral dimension to describe the behavior and to design accelerations of the gossip algorithm. However, the analyses of this paper control only the expected process $\E[x_n]$: the random sampling of the edges is averaged out.

While the polynomial exponents are sharp, we expect the logarithmic factors to be an artifact of the method of proof.

In the case $d=0$ and $V=1$, where no assumption on the structure of the graph is made, the fact that the minimal past energy is $O(n^{-1})$ (neglecting the logarithmic factor) has been noticed by Aldous in \cite[Proposition 4]{aldous2012lecture}. Aldous leaves as an open problem whether one can prove a bound without taking a minimum; this is a special case of our Remark \ref{rmk:pointwise-decay-risk}. 

\subsection{Linear regression with Gaussian features}
\label{sec:gaussian}

In the setting of Section \ref{sec:main-result}, we assume $X$ to be centered Gaussian process of covariance $\Sigma$ where $\Sigma$ is a bounded symmetric semidefinite operator. 
As $X$ is not bounded a.s., we need to use the weaker set of assumptions given in Remark \ref{rmk:weak-assumptions}. We thus need to compute $R_0$ such that $\E\left[\Vert X \Vert^2 X \otimes X\right] \preccurlyeq R_0 \Sigma$ and $\alpha, R_\alpha$ such that $\E\left[\left\langle X , \Sigma^{-\alpha} X\right\rangle  X \otimes X\right] \preccurlyeq R_\alpha \Sigma$. We show here that these conditions are in fact simple trace conditions on $\Sigma$, sometimes called \emph{capacity conditions} \cite{pillaud2018statistical}. 

\begin{lem}
	\label{lem:gaussian-R}
	If $X \sim \cN(0,\Sigma)$ and $A$ is a bounded symmetric operator such that $\Tr(\Sigma A) < \infty$, 
	\begin{align*}
		\E\left[\left\langle X, AX \right\rangle X \otimes X \right] &= 2\Sigma A \Sigma + \Tr(\Sigma A) \Sigma 
		\preccurlyeq \left(2\Vert \Sigma^{1/2} A \Sigma^{1/2} \Vert_{\cH \rightarrow \cH} + \Tr(\Sigma A)\right) \Sigma \, .
	\end{align*}
\end{lem}

\begin{proof}
	Diagonalize $\Sigma = \sum_{i\geq 1} \lambda_i e_i \otimes e_i$. Then there exists independent standard Gaussian random variables $X_i, i\geq 0$ such that $X = \sum_i \lambda_i^{1/2}X_i e_i$. 
	
	Let $i,j \geq 1$. 
	\begin{align*}
		\left\langle e_i, \E\left[\left\langle X,AX\right\rangle X\otimes X\right]e_j \right\rangle &= \E\left[\left\langle X, AX \right\rangle \left\langle e_i, X \otimes X e_j \right\rangle\right] = \E\left[\left\langle X, AX \right\rangle \lambda_i^{1/2}X_i \lambda_j^{1/2}X_j\right] \\
		&= \lambda_i^{1/2} \lambda_j^{1/2} \sum_{k,l} A_{k,l} \lambda_k^{1/2} \lambda_l^{1/2} \E\left[X_iX_jX_kX_l\right] \, .
	\end{align*}
	As $X_i, i\geq 1$ are centered independent random variables, the quantity $\E\left[X_iX_jX_kX_l\right]$ is $0$ in many cases. More precisely, 
	\begin{itemize}
		\item if $i \neq j$, the general term of the sum in non-zero only when $k = i$ and $l=j$ or $k=j$ and $l=i$. This gives
		\begin{equation*}
			\left\langle e_i, \E\left[\left\langle X,AX\right\rangle X\otimes X\right]e_j \right\rangle = 2 A_{i,j}\lambda_i\lambda_j \, .
		\end{equation*}
		\item if $i = j$, the general term of the sum is non-zero only when $k=l$. This gives
		\begin{align*}
			\left\langle e_i, \E\left[\left\langle X,AX\right\rangle X\otimes X\right]e_i \right\rangle &= \lambda_i \sum_k A_{k,k}\lambda_k \E\left[X_i^2X_k^2\right] = \lambda_i \sum_{k\neq i} A_{k,k}\lambda_k + 3 \lambda_i^2 A_{i,i} \\&= \lambda_i \sum_{k} A_{k,k}\lambda_k + 2 \lambda_i^2 A_{i,i} \, .
		\end{align*}
	\end{itemize}
	In both cases, 
	\begin{equation*}
		\left\langle e_i, \E\left[\left\langle X,AX\right\rangle X\otimes X\right]e_j \right\rangle = 2 \lambda_i\lambda_j A_{i,j} + \left(\sum_k A_{k,k}\lambda_k\right) \lambda_i\bfone_{i=j} \, .
	\end{equation*}
	Note that 
	\begin{equation*}
		\Tr(A\Sigma) = \sum_k \left\langle e_k, \Sigma A e_k\right\rangle = \sum_k \lambda_k A_{k,k} \, .
	\end{equation*}
	Thus we get 
	\begin{align*}
		\left\langle e_i, \E\left[\left\langle X,AX\right\rangle X\otimes X\right]e_j \right\rangle &= 2 \lambda_i\lambda_j A_{i,j} + \Tr(A\Sigma) \lambda_i \bfone_{i=j} \\
		&= 2\left\langle e_i, \Sigma A \Sigma e_j \right\rangle + \Tr(A\Sigma) \left\langle e_i, \Sigma e_j \right\rangle \\
		&= \left\langle e_i, \left[2 \Sigma A \Sigma + \Tr(\Sigma A)\Sigma\right]e_j \right\rangle \, . 
	\end{align*}
\end{proof}

From this lemma with $A = \Id$, we compute $R_0 = 2\Vert \Sigma \Vert_{\cH \rightarrow \cH} + \Tr(\Sigma)$, and with $A = \Sigma^{-\alpha}$, we compute $R_\alpha = 2\Vert \Sigma\Vert_{\cH \rightarrow \cH} ^{1-\alpha} + \Tr(\Sigma^{1-\alpha})$. Thus in the Gaussian case, the condition of (weak) regularity of the features is given by $\Tr(\Sigma^{1-\alpha})<\infty$.

\paragraph{Simulations.} We present simulations in finite but large dimension $d = 10^5$, and we check that dimension-independent bounds describe the observed behavior. We artificially generate regression problems with different regularities by varying the decay of the eigenvalues of the covariance $\Sigma$ and varying the decay of the coefficients of $\theta_*$. 

Choose an orthonormal basis $e_1, \dots, e_d$ of $\cH$. We define $\Sigma = \sum_{i= 1}^d i^{-\beta}e_i \otimes e_i$ for some $\beta \geq 1$ and $\theta_* = \sum_{i=1}^d i^{-\delta}e_i$ for some $\delta \geq 1/2$. We now check the condition on $\alpha$ such that the assumptions \ref{ass:reg-opt} and \ref{ass:reg-feature} are satisfied. 
\begin{enumerate}[label = (\alph*), noitemsep, nolistsep]
	\item $\langle \theta_*, \Sigma^{-\alpha} \theta_* \rangle = \sum_{i=1}^d \langle \theta_* , e_i \rangle^2 i^{\beta\alpha} = \sum_{i= 1}^d i^{-2\delta + \alpha\beta}$, which is bounded independently of the dimension $d$ if and only if $\sum_{i= 1}^\infty i^{-2\delta + \alpha\beta} < \infty \Leftrightarrow -2\delta + \alpha\beta < -1 \Leftrightarrow \alpha < \frac{2\delta-1}{\beta}$. 
	\item $\Tr(\Sigma^{1-\alpha}) = \sum_{i=1}^d i^{-\beta(1-\alpha)}$, which is bounded independently of the dimension $d$ if and only if 
	$\sum_{i=1}^\infty i^{-\beta(1-\alpha)} < \infty \Leftrightarrow -\beta(1-\alpha) < -1 \Leftrightarrow \alpha < 1-1/\beta$. 
\end{enumerate} 
Thus the corollary gives dimension-independent convergence rates for all $\alpha < \alpha_* = \min\left(1-\frac{1}{\beta}, \frac{2\delta-1}{\beta}\right)$. 

\begin{figure}
	\includegraphics[width=0.5\textwidth]{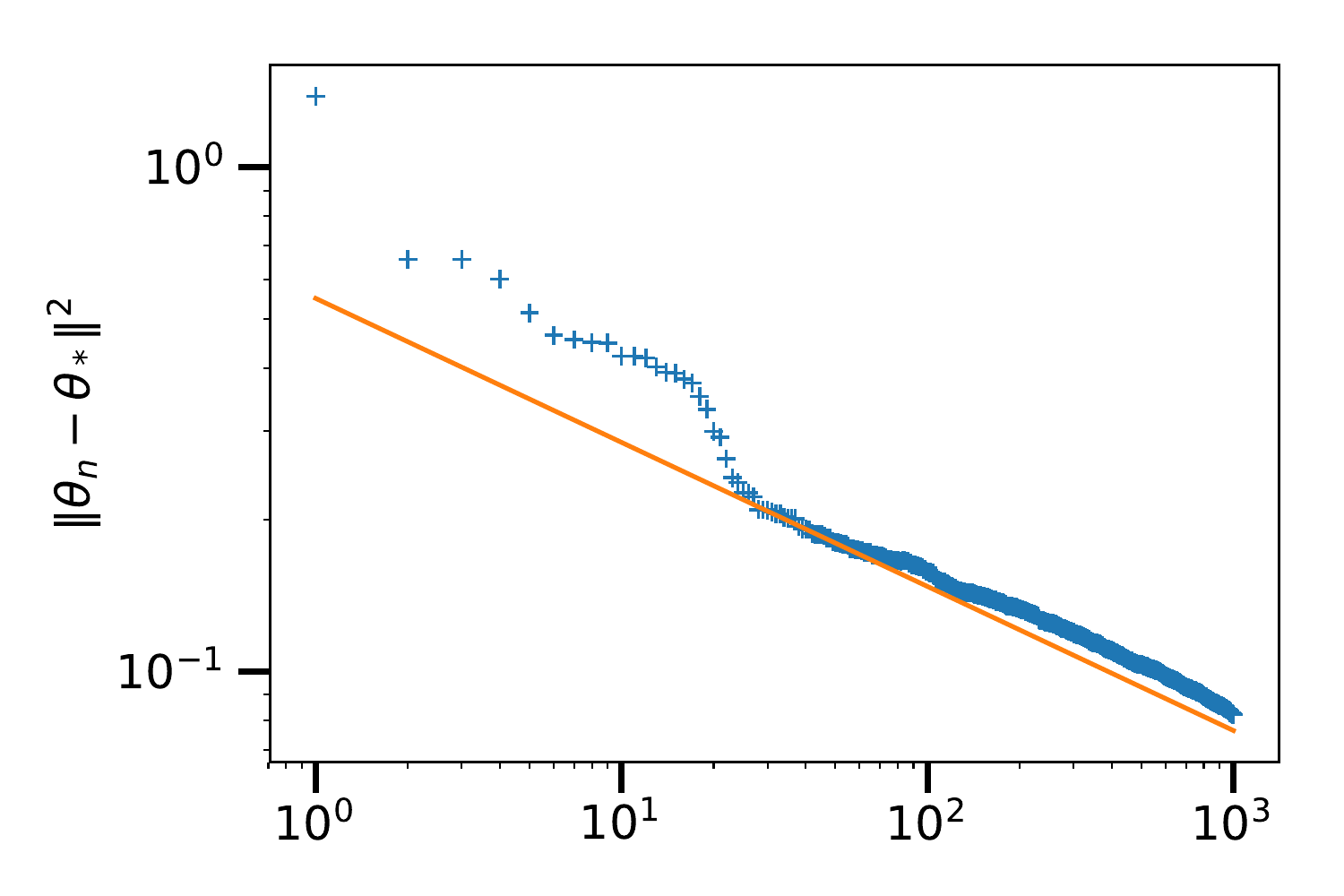}
	\includegraphics[width=0.5\textwidth]{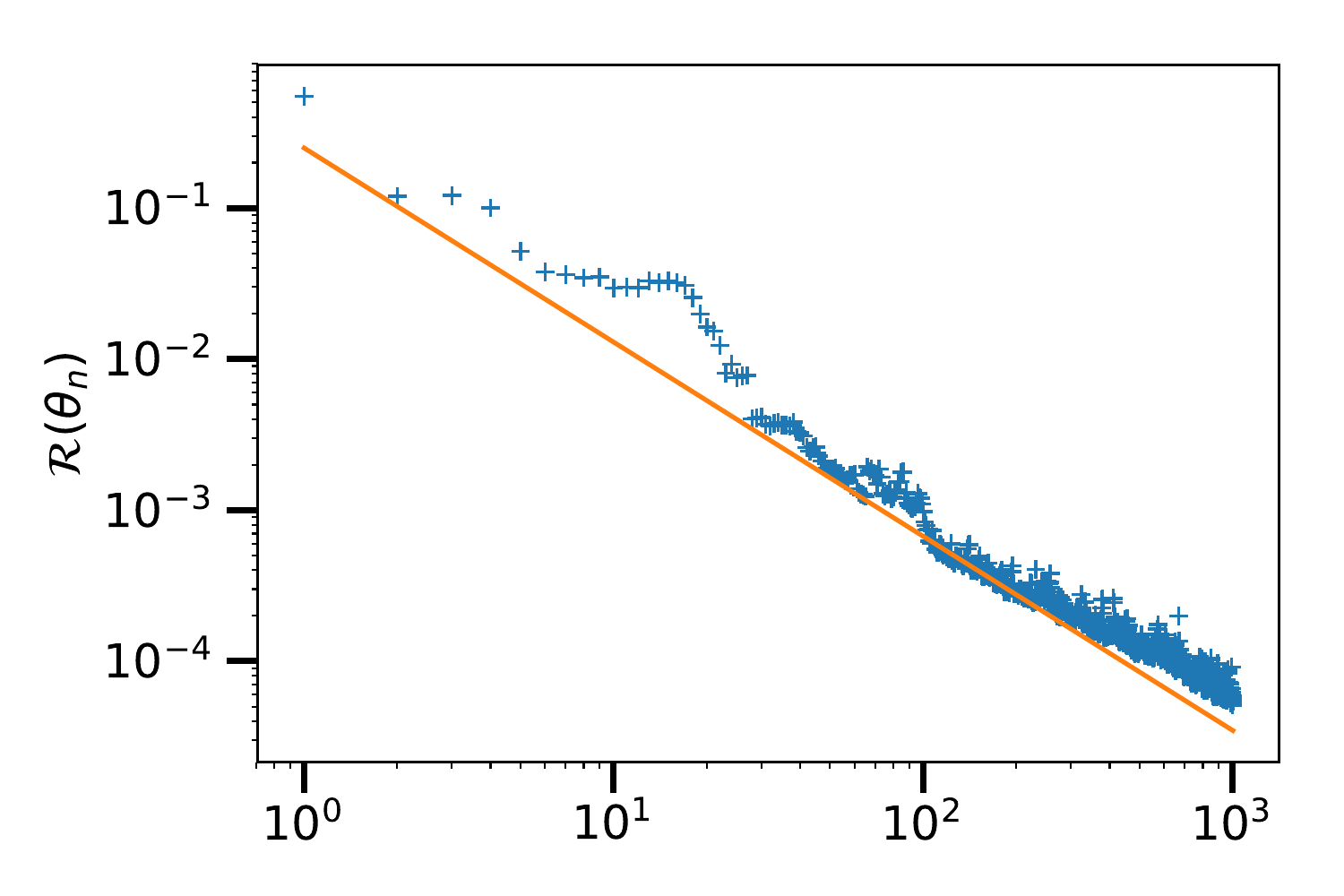}	\includegraphics[width=0.5\textwidth]{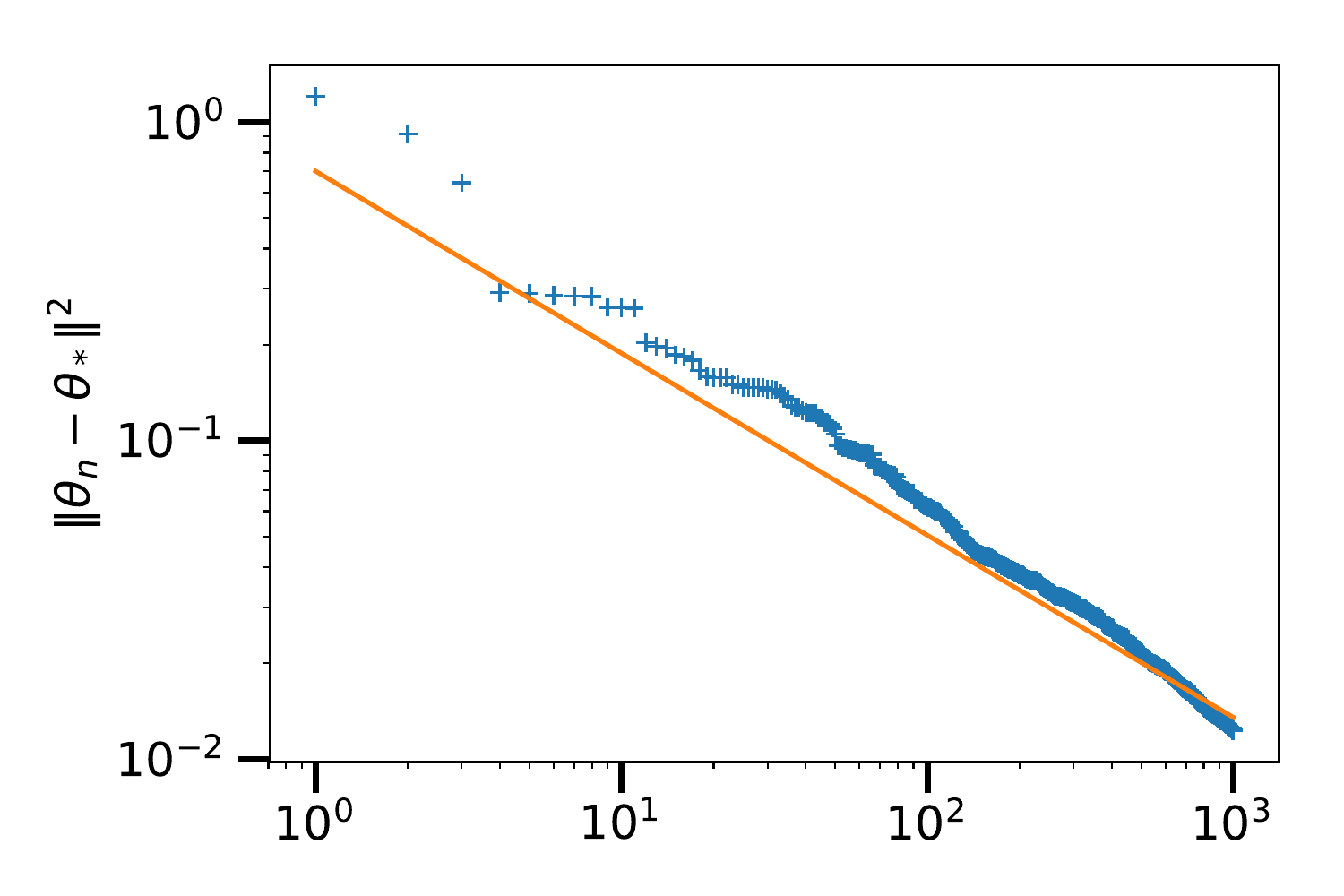}
	\includegraphics[width=0.5\textwidth]{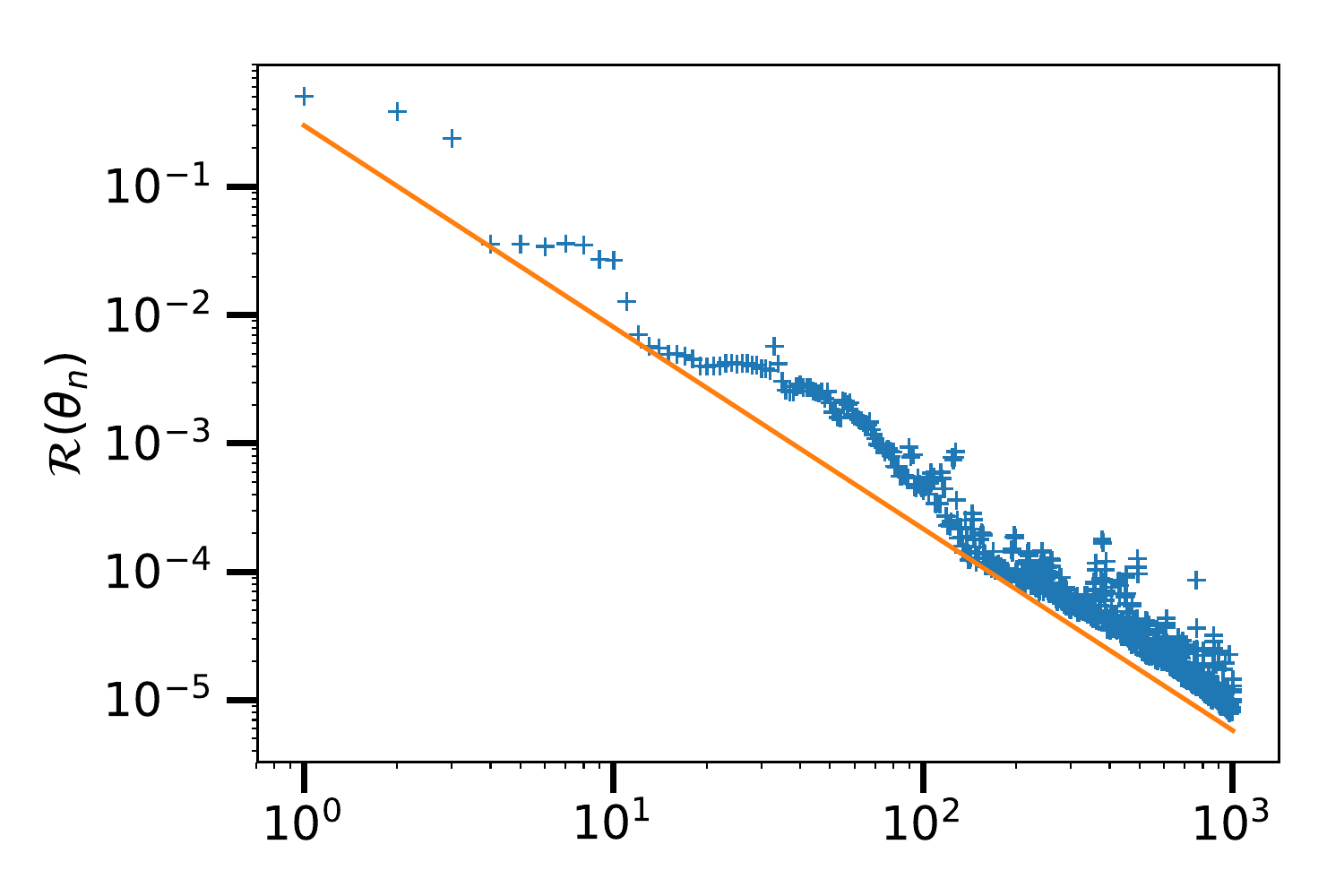}
	\caption{In blue +, evolution of $\Vert \theta_n - \theta_*\Vert^2$ (left) and $\cR(\theta_n)$ (right) as functions of $n$, for the problems with parameters $\beta = 1.4, \delta = 1.2$ (up) and $\beta = 3.5, \delta = 1.5$. The orange lines represent the curves $D/n^{\alpha_*}$ (left) and $D'/n^{\alpha_*+1}$ (right).}
	\label{fig:gaussian}
\end{figure}

In Figure \ref{fig:gaussian}, we show the evolution of $\Vert\theta_n-\theta_*\Vert^2$ and $\cR(\theta_n)$ for two realizations of SGD. We chose the stepsize $\gamma = 1/R_0 = 1/(2\Vert \Sigma \Vert_{\cH \rightarrow \cH} + \Tr(\Sigma))$. The two realizations represent two possible different regimes:
\begin{itemize}[noitemsep, nolistsep]
	\item In the two upper plots, $\beta = 1.4, \delta = 1.2$. The irregularity of the feature vectors is the bottleneck for fast convergence. We have $\alpha_* = \min\left(1-\frac{1}{\beta}, \frac{2\delta-1}{\beta}\right) \approx \min(0.29,1) = 0.29$.  
	\item In the two lower plots, $\beta = 3.5, \delta = 1.5$. The irregularity of the optimum is the bottleneck for fast convergence. We have $\alpha_* = \min\left(1-\frac{1}{\beta}, \frac{2\delta-1}{\beta}\right) \approx \min(0.71,0.57) = 0.57$.  
\end{itemize}

We compare with the curves $D/n^{\alpha_*}$ and $D'/n^{\alpha_*+1}$ with hand-tuned constants $D$ and $D'$ to fit best the data for each plot. In both regimes, our theory is sharp in predicting the exponents in the polynomial rates of convergence of $\Vert\theta_n-\theta_*\Vert^2$ and $\cR(\theta_n)$. 

\section{Robustness to model mispecification}
\label{sec:robustness}

In this section, we describe how the results of Section \ref{sec:general-theory} are perturbed in the case where a linear relation $Y = \langle \theta_*, X \rangle$ a.s.~does not hold. Following the statistical learning framework, we assume a joint law on $(X,Y)$. We further assume that there exists a minimizer $\theta_* \in \cH$ of the population risk $\cR(\theta)$:
\begin{equation*}
	\theta_* \in \argmin_{\theta\in\cH}\left\{ \cR(\theta) = \frac{1}{2}\E\left[\left(Y-\langle \theta, X \rangle \right)^2\right]\right\} \, .
\end{equation*}
This general framework encapsulates two types of perturbations of the noiseless linear model:
\begin{itemize}
	\item (variance) The output $Y$ can be uncertain given $X$. For instance, under the noisy linear model, $Y = \langle \theta_*, X\rangle + Z$, where $Z$ is centered and independent of $X$. In this case, $\cR(\theta_*) = \E[Z^2] = \E[\var(Y|X)]$.
	\item (bias) Even if $Y$ is deterministic given $X$, this dependence can be non-linear: $Y = \psi(X)$ for some non-linear function $\psi$. Then $\cR(\theta_*)$ is the squared $L^2$ distance of the best linear approximation to $\psi$: $\cR(\theta_*) = \frac{1}{2}\E\left[\left(\psi(X)-\langle \theta_*,X \rangle\right)^2\right]$.
\end{itemize}
In the general framework, the optimal population risk is a combination of both sources
\begin{equation*}
	\cR(\theta_*) = \frac{1}{2} \E\left[\var(Y|X)\right] + \frac{1}{2} \E\left[\left(\E[Y|X] - \langle \theta_*, X \rangle\right)^2\right] \, .
\end{equation*}
Given i.i.d.~realizations $(X_1,Y_1),(X_2,Y_2),\dots$ of $(X,Y)$, the SGD iterates are defined as 
\begin{align}
	\label{eq:gen-sgd-iteration}
	&\theta_0 = 0 \, , &&\theta_n = \theta_{n-1} - \gamma \left(\langle \theta_{n-1}, X_n \rangle - Y_n\right) X_n \, .
\end{align}
Apart from the new definition of $\theta_*$, we repeat the same assumptions as in Section \ref{sec:general-theory}: let $R_0 < \infty$ be such that $\Vert X \Vert^2 \leq R_0$ a.s., denote $\Sigma = \E[X\otimes X]$ and $\varphi_n(\beta) = \E\left[\left\langle \theta_n - \theta_*, \Sigma^{-\beta}\left(\theta_n - \theta_*\right)\right\rangle\right]$.

\begin{thm}
	\label{thm:gen-main-result}
	Under the assumptions of Theorem \ref{thm:main-result-upper-bound}, 
	\begin{equation*}
		\min_{k=0, \dots, n } \E\left[\cR(\theta_k) - \cR(\theta_*)\right] \leq 2 \frac{C'}{n^{\alphalow +1}} + 2 R_0 \gamma \cR(\theta_*) \, , 
	\end{equation*}
	where $C'$ is the same constant as in Theorem \ref{thm:main-result-upper-bound}.
\end{thm}

The take-home message is that if we consider the excess risk $\cR(\theta_k) - \cR(\theta_*)$, we get the upper bound of the form $2C'n^{-(\alphalow+1)}$, analog to Theorem \ref{thm:main-result-upper-bound}, but with an additional constant term $2 R_0 \gamma \cR(\theta_*)$. This term can be small if $\cR(\theta_*)$ is small, that is if the problem is close to the noiseless linear model, or if the step-size $\gamma$ is small. In the finite horizon setting setting, one can optimize $\gamma$ as a function of the scheduled number of steps $n$ in order to balance both terms in the upper bound. As $C' \propto \gamma^{-(\alphalow+1)}$, the optimal choice is $\gamma \propto n^{-(\alphalow+1)/(\alphalow+2)}$ which gives a rate $\min_{k=0, \dots, n } \E\left[\cR(\theta_k) - \cR(\theta_*)\right] = O\left(n^{-(\alphalow+1)/(\alphalow+2)}\right)$.

In the theorem below, we study the SGD iterates $\theta_n$ in terms of the power norms $\varphi_n(\beta)$, $\beta \in [-1,\alphalow-1]$, in particular in term of the reconstruction error $\varphi_n(0) = \E[\Vert \theta_n - \theta_* \Vert^2]$ if $\alphalow \geq 1$. Note that the population risk $\cR(\theta)$ is a quadratic with Hessian $\Sigma$, minimized at $\theta_*$, thus
\begin{equation*}
	\E\left[\cR(\theta_n)-\cR(\theta_*)\right] = \frac{1}{2}\E\left[\left\langle \theta_n-\theta_*, \Sigma(\theta_n-\theta_*) \right\rangle \right] = \frac{1}{2}  \varphi_n(-1) \, .
\end{equation*}
Thus the theorem below extends Theorem \ref{thm:gen-main-result}. 

\begin{thm}
	\label{thm:gen-general-result}
	Under the assumptions of Theorem \ref{thm:main-result-upper-bound}, 
	\begin{enumerate}[nolistsep, noitemsep]
		\item for all $\beta \geq 0$, $\beta \leq \alphalow-1$,
		\begin{equation*}
			\varphi_n(\beta) \leq 2\frac{C(\beta)}{n^{\alphalow-\beta}} + 4 R_0^{1-(\beta+1)/\alphalow} R_{\alphalow}^{(\beta+1)/\alphalow}\gamma \cR(\theta_*) \, ,
		\end{equation*}
		\item for all $\beta \in [-1,0)$, $\beta \leq \alphalow-1$,
		\begin{equation*}
			\min_{k01, \dots, n } \varphi_k(\beta) \leq 2\frac{C'(\beta)}{n^{\alphalow-\beta}} + 4 R_0^{1-(\beta+1)/\alphalow} R_{\alphalow}^{(\beta+1)/\alphalow} \gamma \cR(\theta_*) \, ,
		\end{equation*}
	\end{enumerate}
	where $C$, $C'$ are the same constants as in Theorem \ref{thm:general-result-upper-bound}.
\end{thm}

This theorem is proved in Appendix \ref{sec:proof-thm-gen}. We expect the condition $\beta \leq \alphalow-1$ to be necessary. More precisely, when $\cR(\theta_*)$ is positive, we expect the error $\theta_n-\theta_*$ to diverge under the norm $\Vert \Sigma^{-\beta/2} \, . \, \Vert$ if $\beta > \alphalow - 1$. In particular, this would imply that the reconstruction error diverges when $\alphalow < 1$. 

In Figure \ref{fig:gaussian-additive}, we show how the simulations of Section \ref{sec:gaussian} are perturbed in the presence of additive noise. We consider the noisy linear model $Y = \langle \theta_*, X \rangle + \sigma^2 Z$, where $X \sim \cN(0,\Sigma)$ and $Z \sim \cN(0,1)$ are independent. As in the previous simulations, we consider the case $\Sigma = \sum_{i= 1}^d i^{-\beta}e_i \otimes e_i$ and $\theta_* = \sum_{i=1}^d i^{-\delta}e_i$ with here $d = 10^5$, $\beta = 1.4$, $\delta = 1.2$. In the noiseless case $\sigma^2 = 0$, we have shown that the rate of convergence was given by the polynomial exponent $\alpha_* = \min\left(1-\frac{1}{\beta}, \frac{2\delta-1}{\beta}\right)$. These predicted rates are represented by the orange lines in the plots. In blue, we show the results of our simulations with some additive noise with variance $\sigma^2 = 2\times 10^{-4}$. The exponent $\alpha_*$ still describes the behavior of SGD in the initial phase, but in the large $n$ asymptotic the population risk $\cR(\theta_n)$ stagnates around the order of $\sigma^2$. Both of these qualitative behaviors are predicted by Theorem~\ref{thm:gen-main-result}. Moreover, the reconstruction error $\Vert \theta_n - \theta_* \Vert$ diverges for large $n$.

\begin{figure}
	\includegraphics[width=0.5\textwidth]{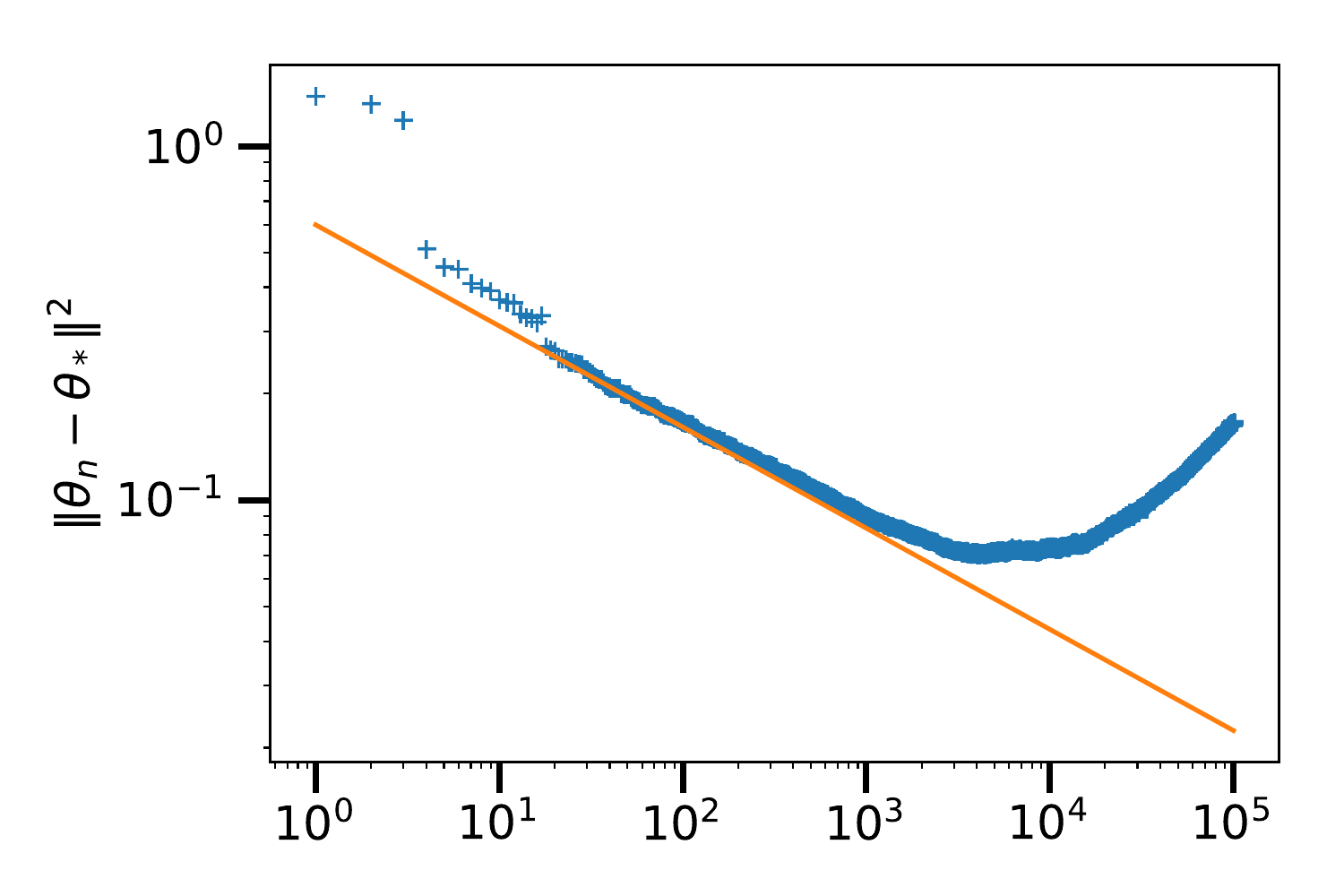}
	\includegraphics[width=0.5\textwidth]{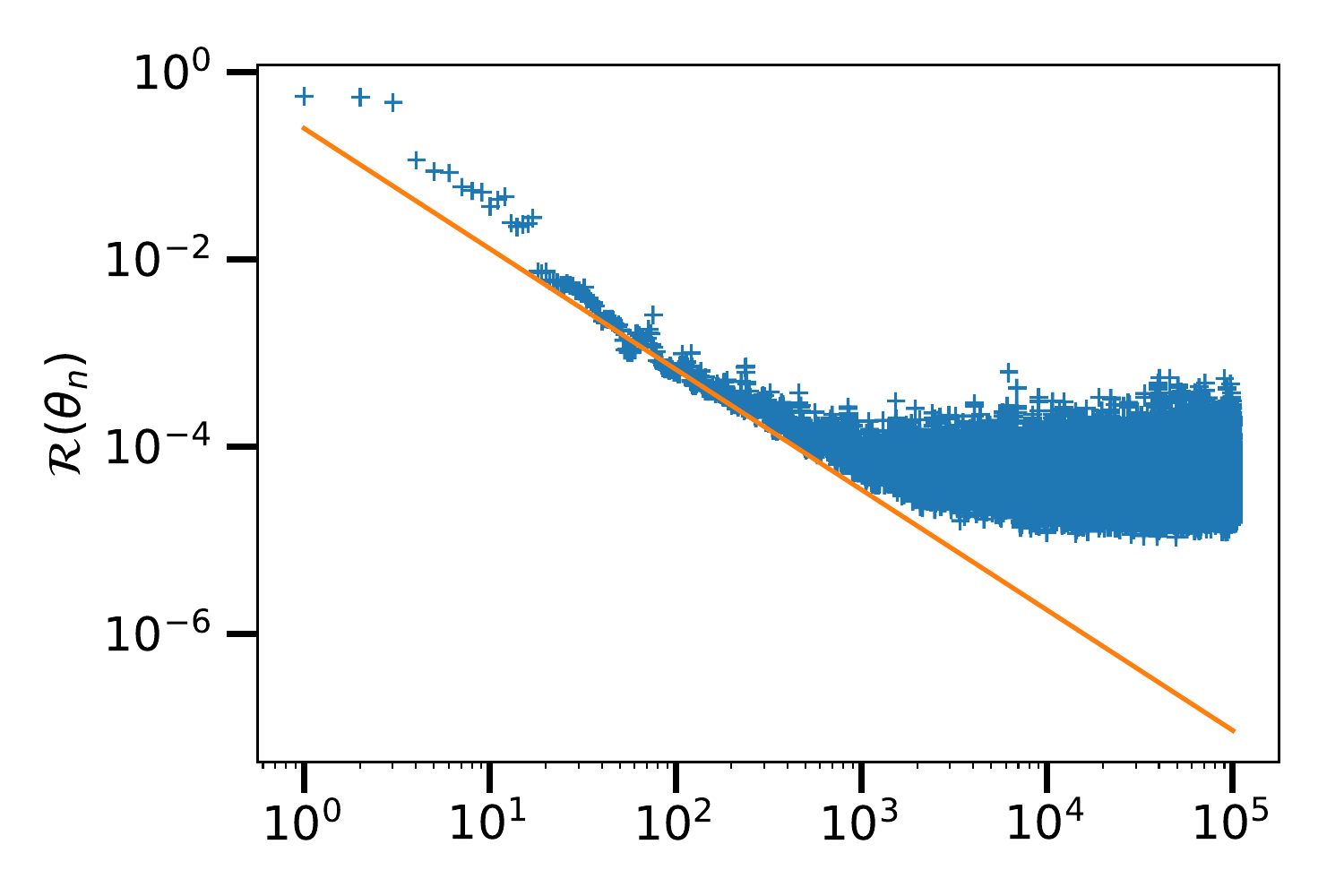}	
	\caption{In blue +, evolution of $\Vert \theta_n - \theta_*\Vert^2$ (left) and $\cR(\theta_n)$ (right) as functions of $n$, for the problems with parameters $d = 10^5, \beta = 1.4, \delta = 1.2$. The orange lines represent the curves $D/n^{\alpha_*}$ (left) and $D'/n^{\alpha_*+1}$ (right).}
	\label{fig:gaussian-additive}
\end{figure}

	\section{Conclusion and research directions}
	
	In this paper, we give a sharp description of the convergence of SGD under the noiseless linear model and made connexions with the interpolation of a real function and the averaging process. The behavior of SGD is surprisingly different in the absence of additive noise: it converges without any averaging or decay of the step-sizes.
	To some extent, SGD adapts to the regularity of the problem thanks to the implicit regularization ensured by the initialization at zero and the single pass on the data. However, by comparing with some known estimators for the interpolation of functions \cite{bauer2017nonparametric, kohler2013optimal} (see the end of Section \ref{sec:intro}), we conjecture that the convergence rate of SGD is suboptimal. What are the minimax rates under the noiseless linear model? Can they be reached with some accelerated online algorithm?

	\section*{Acknowledgments}
	
	This work was greatly improved by detailed comments from Loucas Pillaud-Vivien on earlier versions of the manuscript. We also thank Alessandro Rudi, Nicolas Flammarion and anonymous reviewers for useful discussions. This work was funded in part by the French government under management of Agence Nationale de la
Recherche as part of the ``Investissements d’avenir'' program, reference ANR-19-P3IA-0001 (PRAIRIE 3IA
Institute). We also acknowledge support from the European Research Council (grant SEQUOIA 724063) and from the DGA.


	\bibliographystyle{abbrv}
	\bibliography{biblio}

\newpage

\appendix

\section{Proof of Theorems \ref{thm:main-result-upper-bound} and \ref{thm:general-result-upper-bound}}
\label{sec:proof-thm-general-upper}

We recall here the definition of the regularity functions 
\begin{equation*}
\varphi_n(\beta) = \E\left[\left\langle \theta_n - \theta_*, \Sigma^{-\beta}\left(\theta_n - \theta_*\right)\right\rangle\right] \in [0,\infty] \, , \qquad \beta \in \R \, . 
\end{equation*}

\subsection{Properties of the regularity functions}

We derive here two properties of the sequence of regularity functions $\varphi_n, n \geq 1$ that are useful for the proof of Theorem \ref{thm:general-result-upper-bound}. The first one is a simple consequence of the above definition of the regularity function. The second property is the closed recurrence relation of the regularity functions $\varphi_n$, $n\geq 0$ associated to the iterates of SGD. 

\begin{property}
	\label{prop:log-cvx}
	For all $n$, the function $\varphi_n$ is log-convex, i.e., for all $\beta_1, \beta_2 \in \R$, for all $\lambda \in [0,1]$, 
	\begin{equation*}
	\varphi_n\left((1-\lambda)\beta_1 + \lambda\beta_2\right) \leq \varphi_n(\beta_1)^{1-\lambda} \varphi_n(\beta_2)^\lambda \, .
	\end{equation*}
\end{property}

\begin{proof}
	The proof is based on the following lemma, that we state clearly for another use below. 
	\begin{lem}
		\label{lem:holder-for-operators}
		Let $\theta\in\cH$. Then for all $\beta_1, \beta_2 \in \R$, $\lambda \in [0,1]$, 
		\begin{equation*}
		\left\langle \theta, \Sigma^{-[(1-\lambda)\beta_1 + \lambda\beta_2]}\theta \right\rangle \leq \left\langle \theta, \Sigma^{-\beta_1}\theta \right\rangle^{1-\lambda}  \left\langle \theta, \Sigma^{-\beta_2}\theta \right\rangle^{\lambda} \, . 
		\end{equation*}
	\end{lem}
	This lemma follows from H\"older's inequality with $p = (1-\lambda)^{-1}$ and $q = \lambda^{-1}$. Indeed, diagonalize $\Sigma = \sum_{i} \mu_i e_i \otimes e_i$. Then 
	\begin{align*}
	\left\langle \theta, \Sigma^{-[(1-\lambda)\beta_1 + \lambda\beta_2]}\theta \right\rangle &= \sum_i \mu_i^{-[(1-\lambda)\beta_1 + \lambda\beta_2]} \langle \theta, e_i \rangle^2 \\
	&= \sum_i \left(\mu_i^{-\beta_1}\langle \theta, e_i \rangle^2\right)^{1-\lambda}\left(\mu_i^{-\beta_2}\langle \theta, e_i \rangle^2\right)^{\lambda} \\
	&\leq \left(\sum_i \mu_i^{-\beta_1}\langle \theta, e_i \rangle^2\right)^{1-\lambda} \left(\sum_i \mu_i^{-\beta_2}\langle \theta, e_i \rangle^2\right)^{\lambda} \\
	&= \left\langle \theta, \Sigma^{-\beta_1}\theta \right\rangle^{1-\lambda}  \left\langle \theta, \Sigma^{-\beta_2}\theta \right\rangle^{\lambda} \, .
	\end{align*}
	We now apply this lemma to prove Property \ref{prop:log-cvx}. 
	\begin{align*}
	\varphi_n((1-\lambda)\beta_1+\lambda\beta_2) &=  \E\left[\left\langle \theta_n - \theta_*, \Sigma^{-[(1-\lambda)\beta_1+\lambda\beta_2]}\left(\theta_n - \theta_*\right)\right\rangle\right] \\
	&\leq \E\left[\left\langle \theta_n - \theta_*, \Sigma^{-\beta_1}\left(\theta_n - \theta_*\right)\right\rangle^{1-\lambda}\left\langle \theta_n - \theta_*, \Sigma^{-\beta_2}\left(\theta_n - \theta_*\right)\right\rangle^{\lambda}\right] \, .
	\end{align*}
	Using again H\"older's inequality, we get 
	\begin{align*}
	\varphi_n((1-\lambda)\beta_1+\lambda\beta_2) &\leq   \E\left[\left\langle \theta_n - \theta_*, \Sigma^{-\beta_1}\left(\theta_n - \theta_*\right)\right\rangle\right]^{1-\lambda}\E\left[\left\langle \theta_n - \theta_*, \Sigma^{-\beta_2}\left(\theta_n - \theta_*\right)\right\rangle\right]^{\lambda} \\
	&=\varphi_n(\beta_1)^{1-\lambda}\varphi_n(\beta_2)^\lambda \, .
	\end{align*}
\end{proof}

\begin{property}
	\label{prop:regularity-ineq}
	Under the assumptions of Theorem \ref{thm:general-result-upper-bound}, for all $n$, the function $\varphi_n$ is finite on $(-\infty,\alphalow]$, and if $0 \leq \beta \leq \alphalow$,
	\begin{equation*}
	\varphi_n(\beta) \leq \varphi_{n-1}(\beta) - 2\gamma \varphi_{n-1}(\beta-1) + \gamma^2  R_0^{1-\beta/\alphalow}R_\alphalow^{{\beta}/{\alphalow}}\varphi_{n-1}(-1) \, .
	\end{equation*}
\end{property}

\begin{proof}
	By assumption \ref{ass:reg-opt}, $\varphi_0(\alphalow) =  \Vert \Sigma^{-\alphalow/2} \theta_* \Vert^2$ is finite, i.e., there exists $\theta \in \cH$ such that $\theta_* = \Sigma^{\alphalow/2} \theta$. Then for any $\beta \leq \alphalow$, $\theta_* = \Sigma^{\beta/2}\left(\Sigma^{(\alphalow - \beta)/2} \theta\right)$ thus $\varphi_0(\beta) = \Vert \Sigma^{-\beta/2} \theta_* \Vert^2$ is finite. 
	
	Further, assume that for some $n$, the function $\varphi_{n-1}$ is finite on $(\infty,\alphalow]$. Then 
	we can rewrite the stochastic gradient iteration \eqref{eq:SGD_iteration} as 
	\begin{equation*}
	\theta_n - \theta_* = (\Id - \gamma X_n \otimes X_n)(\theta_{n-1} - \theta_*) \, .
	\end{equation*}
	Substituting this expression in the definition of $\varphi_n$ and expanding the formula, we get
	\begin{align}
	\varphi_n(\beta) &=  \E\left[\left\langle \theta_n - \theta_*, \Sigma^{-\beta}\left(\theta_n - \theta_*\right)\right\rangle\right] \nonumber \\ 
	&= \E\left[\left\langle (\Id - \gamma X_n \otimes X_n)(\theta_{n-1} - \theta_*), \Sigma^{-\beta}(\Id - \gamma X_n \otimes X_n)(\theta_{n-1} - \theta_*)\right\rangle\right] \nonumber\\
	&= \E\left[\left\langle \theta_{n-1} - \theta_*, \Sigma^{-\beta}(\theta_{n-1} - \theta_*)\right\rangle\right]\label{eq:aux-2}  \\
	&\qquad\qquad\qquad-2\gamma \E\left[\left\langle \theta_{n-1} - \theta_*, \Sigma^{-\beta} X_n \otimes X_n(\theta_{n-1} - \theta_*)\right\rangle\right] \\
	&\qquad\qquad\qquad +\gamma^2\E\left[\left\langle \theta_{n-1} - \theta_*,  X_n \otimes X_n\Sigma^{-\beta} X_n \otimes X_n(\theta_{n-1} - \theta_*)\right\rangle\right] \label{eq:aux-3}\, .
	\end{align}
	Note that the first term of this sum is $\varphi_{n-1}(\beta)$. Further, $\theta_{n-1}$ is computed using only $(X_1,Y_1), \dots, (X_{n-1}, Y_{n-1})$, thus it is independent of $X_n$. It follows that 
	\begin{align}
	\E\left[\left\langle \theta_{n-1} - \theta_*, \Sigma^{-\beta} X_n \otimes X_n(\theta_{n-1} - \theta_*)\right\rangle\right] &= \E\left[\left\langle \theta_{n-1} - \theta_*, \Sigma^{-\beta} \E\left[X_n \otimes X_n\right](\theta_{n-1} - \theta_*)\right\rangle\right] \nonumber \\
	&= \E\left[\left\langle \theta_{n-1} - \theta_*, \Sigma^{-\beta+1}(\theta_{n-1} - \theta_*)\right\rangle\right] \nonumber \\
	&= \varphi_{n-1}(\beta-1) \label{eq:aux-4}\, .
	\end{align}
	Finally,
	\begin{align}
	&\E\left[\left\langle \theta_{n-1} - \theta_*,  X_n \otimes X_n\Sigma^{-\beta} X_n \otimes X_n(\theta_{n-1} - \theta_*)\right\rangle\right] \label{eq:aux-1} \\
	&\qquad = \E\left[\left\langle \theta_{n-1}-\theta_*, X_n \right\rangle^2 \left\langle X_n, \Sigma^{-\beta}X_n \right\rangle \right] \label{eq:aux-11} 
	\end{align} 
	We now assume that $0 \leq \beta \leq \alphalow$. We apply Lemma \ref{lem:holder-for-operators} with $\beta_1 = 0, \beta_2 = \alphalow, \lambda = \beta/\alphalow$: 
	\begin{equation*}
	\left\langle X_n, \Sigma^{-\beta}X_n \right\rangle \leq \Vert X_n \Vert^{2(1-\beta/\alphalow)}\left\langle X_n, \Sigma^{-\alphalow}X_n \right\rangle^{\beta/\alphalow}
	\end{equation*}
	Let $E_{X_n}$ denote the expectation with respect to $X_n$ only, while keeping $X_0, \dots, X_{n-1}$ random. Applying H\"older's inequality, we get 
	\begin{align*}
&\E_{X_n}\left[\left\langle X_n, \Sigma^{-\beta}X_n \right\rangle  \left\langle \theta_{n-1}-\theta_*, X_n \right\rangle^2 \right] \\
&\qquad \leq  \E_{X_n}\left[\Vert X_n \Vert^{2(1-\beta/\alphalow)}\left\langle X_n, \Sigma^{-\alphalow}X_n \right\rangle^{\beta/\alphalow} \left\langle \theta_{n-1}-\theta_*, X_n \right\rangle^2 \right] \\
	&\qquad\leq  \E_{X_n}\left[\Vert X_n \Vert^{2} \left\langle \theta_{n-1}-\theta_*, X_n \right\rangle^2 \right]^{1-\beta/\alphalow}\E\left[\left\langle X_n, \Sigma^{-\alphalow}X_n \right\rangle \left\langle \theta_{n-1}-\theta_*, X_n \right\rangle^2 \right]^{\beta/\alphalow} \\
	&\qquad= \left\langle \theta_{n-1}-\theta_*, \E\left[\Vert X_n \Vert^2 X_n \otimes X_n\right](\theta_{n-1}-\theta_*) \right\rangle^{1-\beta/\alphalow}  \\
	&\qquad\qquad\qquad\times\left\langle \theta_{n-1}-\theta_*, \E\left[\left\langle X_n, \Sigma^{-\alphalow}X_n \right\rangle X_n \otimes X_n\right](\theta_{n-1}-\theta_*) \right\rangle^{\beta/\alphalow} \\
	&\qquad\leq R_0^{1-\beta/\alphalow}R_\alphalow^{\beta/\alphalow} \left\langle \theta_{n-1}-\theta_*, \Sigma(\theta_{n-1}-\theta_*) \right\rangle \, ,
	\end{align*}
	where in this last step, we use the assumptions that the features $X$ are bounded and regular, in their weak formulation of Remark \ref{rmk:weak-assumptions}. Returning to the computation of \eqref{eq:aux-1}-\eqref{eq:aux-11}, we get 
	\begin{align}
	&\E\left[\left\langle \theta_{n-1} - \theta_*,  X_n \otimes X_n\Sigma^{-\beta} X_n \otimes X_n(\theta_{n-1} - \theta_*)\right\rangle\right] \nonumber\\
	&\qquad=\E\left[\E_{X_n}\left[\left\langle \theta_{n-1}-\theta_*, X_n \right\rangle^2 \left\langle X_n, \Sigma^{-\beta}X_n \right\rangle \right]\right] \\
	&\qquad\leq R_0^{1-\beta/\alphalow}R_\alphalow^{\beta/\alphalow} \E\left[\left\langle \theta_{n-1}-\theta_*, \Sigma(\theta_{n-1}-\theta_*) \right\rangle \right] \nonumber \\
	&\qquad= R_0^{1-\beta/\alphalow}R_\alphalow^{\beta/\alphalow}  \varphi_{n-1}(-1) \label{eq:aux-5}\, .
	\end{align}
	The result is obtained by putting together Equations \eqref{eq:aux-2}-\eqref{eq:aux-3}, \eqref{eq:aux-4} and \eqref{eq:aux-5}.
\end{proof}

\subsection{Proof of Theorem \ref{thm:main-result-upper-bound}}

A remarkable feature of the proof that follows is that only Properties \ref{prop:log-cvx} and \ref{prop:regularity-ineq} of the regularity functions are used to derive the theorem. In particular, we do not use the definition of the regularity functions $\varphi_n$ in this section.   

We start with a few preliminary remarks. Using the recurrence Property \ref{prop:regularity-ineq} and that $\gamma R_0 \leq 1$, 
\begin{align*}
\varphi_k(0) &\leq \varphi_{k-1}(0) - \gamma \left(2 - \gamma R_0\right) \varphi_{k-1}(-1) \\
&\leq \varphi_{k-1}(0) - \gamma  \varphi_{k-1}(-1) \, .
\end{align*}
Thus the sequence $\varphi_k(0)$, $k\geq 0$ decreases, and 
\begin{equation}
\label{eq:bound--1}
\gamma  \varphi_{k-1}(-1) \leq \varphi_{k-1}(0) - \varphi_{k}(0) \, .
\end{equation}
By summing this inequality over $k \geq 1$, we get 
\begin{equation}
\label{eq:bound_sum_-1}
\gamma \sum_{k=0}^\infty  \varphi_{k}(-1) \leq \varphi_0(0) \, .
\end{equation}
Using again the recurrence Property \ref{prop:regularity-ineq}, 
\begin{align}
\varphi_k(\alphalow) &\leq \varphi_{k-1}(\alphalow) - 2 \gamma  \varphi_{k-1}(\alphalow-1) + \gamma^2  R_\alphalow \varphi_{k-1}(-1) \label{eq:aux-7} \\
&\leq \varphi_{k-1}(\alphalow)  + \gamma^2  R_\alphalow \varphi_{k-1}(-1) \, . \nonumber
\end{align}
By summing for $k = 1, \dots, n$ and using the bound \eqref{eq:bound_sum_-1},
\begin{align}
\varphi_n(\alphalow) &\leq \varphi_0(\alphalow) + \gamma^2  R_\alphalow\sum_{k=0}^{n-1} \varphi_k(-1) \nonumber \\
&\leq \varphi_0(\alphalow) +  \gamma R_\alphalow \varphi_0(0) \nonumber \\
&\leq  \varphi_0(\alphalow) +   \frac{R_\alphalow}{R_0} \varphi_0(0) \, . \label{eq:bound-alpha} 
\end{align}
In words, the sequence $\varphi_n(\alphalow)$, $n \geq 0$ is bounded by $D := \varphi_0(\alphalow) +   \frac{R_\alphalow}{R_0} \varphi_0(0)$. As a side note, this proves Theorem \ref{thm:general-result-upper-bound} for $\beta = \alphalow$. 

We can now give a closed recurrence relation $\varphi_k(0)$, $k \geq 0$. Using the log-convexity Property \ref{prop:log-cvx}, 
\begin{equation*}
\varphi_{k-1}(0) \leq  \varphi_{k-1}(-1)^{\alphalow/(\alphalow +1)} \varphi_{k-1}(\alphalow)^{1/(\alphalow +1)} \leq  \varphi_{k-1}(-1)^{\alphalow/(\alphalow +1)} D^{1/(\alphalow +1)}  \, .
\end{equation*}
Substituting in \eqref{eq:bound--1}, we obtain 
\begin{align*}
  \varphi_{k-1}(0) - \varphi_{k}(0) &\geq \gamma\varphi_{k-1}(-1) \\
  &\geq \gamma D^{-1/\alphalow} \varphi_{k-1}(0)^{1+1/\alphalow} \, .
\end{align*}
This gives the wanted closed recurrence relation for $\varphi_k(0)$, $k \geq 0$. It implies a decay of $\varphi_k(0)$ as follows: consider the real function $f(\varphi) = \frac{1}{\varphi^{1/\alphalow}}$. It is a convex function on the positive reals, with derivative $f'(\varphi) = -\frac{1}{\alphalow} \frac{1}{\varphi^{1+1/\alphalow}}$. Using that a convex function is above its tangents, we obtain 
\begin{align*}
f\left(\varphi_k(0)\right) - f\left(\varphi_{k-1}(0)\right) &\geq f'\left(\varphi_{k-1}(0)\right) \left( \varphi_k(0) - \varphi_{k-1}(0) \right) \\
&= -\frac{1}{\alphalow} \frac{1}{\varphi_{k-1}(0)^{1+1/\alphalow}}\left( \varphi_k(0) - \varphi_{k-1}(0) \right) \\
&\geq \frac{1}{\alphalow} \gamma D^{-1/\alphalow} \, .
\end{align*}
By summing this inequality for $k = 1, \dots, n$, we obtain
\begin{align*}
\frac{1}{\varphi_n(0)^{1/\alphalow}} = f\left(\varphi_n(0)\right) \geq f\left(\varphi_0(0)\right) + \frac{1}{\alphalow} \gamma D^{-1/\alphalow} n \geq \frac{1}{\alphalow} \gamma D^{-1/\alphalow} n \, .
\end{align*}
This implies conclusion \ref{concl:positive} of Theorem \ref{thm:main-result-upper-bound}:
\begin{align}
\E\left[\Vert \theta_n - \theta_* \Vert^2 \right] = \varphi_n(0) \leq \frac{\alphalow^{\alphalow}}{\gamma^{\alphalow}} D  \frac{1}{n^{\alphalow}} \, . \label{eq:bound-0}
\end{align}

Further,
\begin{equation*}
\min_{0 \leq k \leq n} \varphi_k(-1) \leq \min_{ \left\lceil n/2 \right\rceil \leq k \leq n} \varphi_k(-1) \leq \frac{2}{n} \sum_{k = \left\lceil n/2 \right\rceil}^{n} \varphi_k(-1) \leq \frac{2}{n} \frac{1}{\gamma} \sum_{k = \left\lceil n/2 \right\rceil}^{n} \left(\varphi_{k}(0) - \varphi_{k+1}(0)\right) \, ,
\end{equation*}
where in the last step we used \eqref{eq:bound--1}. Telescoping the sum, we obtain 
\begin{align}
\min_{0 \leq k \leq n} \varphi_k(-1) &\leq \min_{ \left\lceil n/2 \right\rceil \leq k \leq n} \varphi_k(-1) \leq \frac{2}{n} \frac{1}{\gamma} \varphi_{\left\lceil n/2 \right\rceil}(0) \label{eq:aux-6} \\
&\leq \frac{2}{n} \frac{1}{\gamma} \frac{\alphalow^\alphalow}{\gamma^\alphalow} D \frac{1}{\left\lceil n/2 \right\rceil^{\alphalow}} \leq 2^{\alphalow+1} \frac{\alphalow^\alphalow}{\gamma^{\alphalow+1}} D \frac{1}{n^{\alphalow+1}} \, . \nonumber 
\end{align}
Using that $\varphi_n(-1) = 2\E[\cR(\theta_n)]$, this gives conclusion \ref{concl:negative} of Theorem \ref{thm:main-result-upper-bound}.

\subsection{Proof of Theorem \ref{thm:general-result-upper-bound}}

We continue the proof of Theorem \ref{thm:main-result-upper-bound} to prove Theorem \ref{thm:general-result-upper-bound}. By the log-convexity Property \ref{prop:log-cvx}, for all $\beta \in [0,\alphalow]$,
\begin{equation*}
\varphi_n(\beta) \leq \varphi_n(0)^{1-\beta/\alphalow} \varphi_n(\alphalow)^{\beta/\alphalow} \, . 
\end{equation*}
Using Equations \eqref{eq:bound-0} and \eqref{eq:bound-alpha}, we obtain 
\begin{equation*}
\varphi_n(\beta) \leq \frac{\alphalow^{\alphalow-\beta}}{\gamma^{\alphalow-\beta}} D \frac{1}{n^{\alphalow-\beta}} \, . 
\end{equation*}
This proves conclusion \ref{concl:positive} of the theorem. We now consider the case $\beta \in [-1,0)$. By the log-convexity Property \ref{prop:log-cvx},
\begin{equation*}
\min_{0 \leq k \leq n} \varphi_k(\beta) \leq \min_{\left\lceil n/2 \right\rceil \leq k \leq n} \varphi_k(\beta) \leq \min_{\left\lceil n/2 \right\rceil \leq k \leq n} \varphi_k(-1)^{-\beta} \varphi_k(0)^{1+\beta}
\end{equation*}
Using that $\varphi_k(0)$, $k \geq 0$ is decreasing and the inequality \eqref{eq:aux-6}, we obtain
\begin{align*}
\min_{\left\lceil n/2 \right\rceil \leq k \leq n} \varphi_k(-1)^{-\beta} \varphi_k(0)^{1+\beta} &\leq \varphi_{\left\lceil n/2 \right\rceil}(0)^{1+\beta}\left(\min_{\left\lceil n/2 \right\rceil \leq k \leq n} \varphi_k(-1)\right)^{-\beta} \\
&\leq \varphi_{\left\lceil n/2 \right\rceil}(0)^{1+\beta}\left(\frac{2}{n} \frac{1}{\gamma} \varphi_{\left\lceil n/2 \right\rceil}(0)\right)^{-\beta} \\
&\leq \frac{2^{-\beta}}{n^{-\beta}}\frac{1}{\gamma^{-\beta}} \varphi_{\left\lceil n/2 \right\rceil}(0) \, .
\end{align*}
Using finally \eqref{eq:bound-0}, we obtain conclusion \ref{concl:negative} of the theorem
\begin{equation*}
\min_{0 \leq k \leq n} \varphi_k(\beta) \leq \frac{2^{-\beta}}{n^{-\beta}}\frac{1}{\gamma^{-\beta}} \frac{\alphalow^\alphalow}{\gamma^\alphalow} D \frac{1}{\left\lceil n/2 \right\rceil^{\alphalow}} \leq 2^{\alphalow - \beta} \frac{\alphalow^{\alphalow}}{\gamma^{\alphalow-\beta}} D \frac{1}{n^{\alphalow-\beta}} \, .
\end{equation*}

\section{Proof of Theorems \ref{thm:main-result-lower-bound} and \ref{thm:general-result-lower-bound}}
\label{sec:proof-thm-general-lower}

We start in the case \ref{ass:reg-opt} where the optimum is irregular: $\theta_* \notin \Sigma^{-\alphahigh/2}(\cH)$. In that case, we give a lower bound in the convergence rate by studying the expected process $\overline{\theta}_n := \E[\theta_n]$. Indeed, by Jensen's inequality, 
\begin{equation}
\label{eq:lower-bound-expected-process}
\varphi_n(\beta) = \E\left[\left\langle \theta_n - \theta_*, \Sigma^{-\beta}\left(\theta_n - \theta_*\right)\right\rangle\right] \geq \left\langle \overline{\theta}_n - \theta_*, \Sigma^{-\beta}\left(\overline{\theta}_n - \theta_*\right)\right\rangle \, .
\end{equation}
The expectation $\overline{\theta}_n$ can be interpreted as the (non-stochastic) gradient descent on the population risk $\cR(\theta)$. Indeed, by taking the expectation in \eqref{eq:SGD_iteration}, we obtain
\begin{equation}
\label{eq:expected-iteration}
\overline{\theta}_n - \theta_* = (\Id-\gamma\Sigma)(\overline{\theta}_{n-1}-\theta_*) = -(\Id-\gamma\Sigma)^{n}\theta_* \, .
\end{equation}
Note that as $\gamma \leq 1/R_0$, $I-\gamma\Sigma$ is a positive definite matrix. Indeed, by the weak definition of $R_0$ in Remark \ref{rmk:weak-assumptions},
\begin{equation*}
R_0 \Sigma \succcurlyeq \E\left[\Vert X \Vert^2 X \otimes X \right] = \E\left[( X \otimes X)( X \otimes X) \right] \succcurlyeq \E[X \otimes X]^2 = \Sigma^2 \, ,
\end{equation*}
thus $R_0$ is larger than the operator norm of $\Sigma$. Thus $\gamma \Sigma \preccurlyeq \frac{1}{R_0} \Sigma \preccurlyeq  \Id$. 

In the following, if $\alpha \in \R$ and $k \in \N$, $\binom{\alpha}{k}$ denotes the generalized binomial coefficient: $\binom{\alpha}{k} = \frac{\alpha(\alpha-1)\cdots(\alpha-k+1)}{k!}$. Fix now $\alpha \geq 0$. We have the (formal) power series
\begin{align*}
(1+x)^{-\alpha} &= \sum_{k=0}^{\infty} \binom{-\alpha}{k} x^k \\
(1-x)^{-\alpha} &= \sum_{k=0}^{\infty} \binom{-\alpha}{k} (-1)^k x^k =  \sum_{k=0}^{\infty} \binom{\alpha+k-1}{k} x^k \\ 
y^{-\alpha} &= \sum_{k=0}^{\infty} \binom{\alpha+k-1}{k} (1-y)^k \, .
\end{align*}
This last equality holds in $[0,\infty]$ for $y \in [0,1]$. In that case, all terms of the serie are positive, thus the meaning of the sum is unambiguous. 

Note that $0 \preccurlyeq \gamma \Sigma \preccurlyeq \Id$, thus we have, formally, 
\begin{equation*}
\gamma^{-\alpha} \Sigma^{-\alpha} = \sum_{k=0}^{\infty} \binom{\alpha+k-1}{k}(\Id - \gamma\Sigma)^k \, .
\end{equation*}
The rigorous meaning of this equality is that for all $\theta\in \cH$, 
\begin{equation*}
\gamma^{-\alpha} \langle \theta, \Sigma^{-\alpha} \theta \rangle = \sum_{k=0}^{\infty} \binom{\alpha+k-1}{k} \langle \theta, (\Id - \gamma \Sigma)^k \theta \rangle \, .
\end{equation*}
Both terms of the equality can be infinite: here we are using the convention stated in Section \ref{sec:main-result} that implies that $\langle \theta, \Sigma^{-\alpha}\theta \rangle = \infty \Leftrightarrow \theta \notin \Sigma^{\alpha/2}(\cH)$. In particular, take $\alpha = \alphahigh - \beta$ and $\theta = \Sigma^{-\beta/2}\theta_*$:
\begin{align*}
\infty &= \gamma^{\beta-\alphahigh} \left\langle \theta_*, \Sigma^{-\alphahigh} \theta_* \right\rangle = \sum_{k=0}^\infty \binom{\alphahigh-\beta+k-1}{k} \left\langle \theta_*, \Sigma^{-\beta} (\Id - \gamma \Sigma)^k \theta_* \right\rangle \\
&= \sum_{n=0}^\infty \bigg[\binom{\alphahigh-\beta+2n-1}{2n} \left\langle \theta_*, \Sigma^{-\beta} (\Id - \gamma \Sigma)^{2n} \theta_* \right\rangle \\
&\hspace{2cm}+ \binom{\alphahigh-\beta+2n}{2n+1} \left\langle \theta_*, \Sigma^{-\beta} (\Id - \gamma \Sigma)^{2n+1} \theta_* \right\rangle  \bigg] \, .
\end{align*} 
Using that $ \binom{\alphahigh-\beta+2n-1}{2n} \leq \binom{\alphahigh-\beta+2n}{2n+1}$ and $\left\langle \theta_*, \Sigma^{-\beta} (\Id - \gamma \Sigma)^{2n} \theta_* \right\rangle \geq \left\langle \theta_*, \Sigma^{-\beta} (\Id - \gamma \Sigma)^{2n+1} \theta_* \right\rangle$ and then \eqref{eq:expected-iteration}, \eqref{eq:lower-bound-expected-process},  
\begin{align*}
\infty &\leq 2 \sum_{n=0}^{\infty} \binom{\alphahigh-\beta+2n}{2n+1}\left\langle \theta_*, \Sigma^{-\beta} (\Id - \gamma \Sigma)^{2n} \theta_* \right\rangle \\ 
&= 2 \sum_{n=0}^{\infty} \binom{\alphahigh-\beta+2n}{2n+1}\left\langle \overline{\theta}_n - \theta_*, \Sigma^{-\beta} (\overline{\theta}_n - \theta_*) \right\rangle \\
&\leq 2 \sum_{n=0}^{\infty} \binom{\alphahigh-\beta+2n}{2n+1} \varphi_n(\beta) \, .
\end{align*}
From \cite[Equation 5.8.1]{NIST:DLMF}, we have the formula $\Gamma(z) = \lim_{k\to\infty}\frac{k!k^z}{z(z+1)\cdots(z+k)}$ where $\Gamma$ denotes the Gamma function. Thus as $n \to \infty$ 
\begin{equation*}
\binom{\alphahigh-\beta+2n}{2n+1} = \frac{(\alphahigh-\beta)(\alphahigh-\beta+1)\cdots(\alphahigh-\beta+2n)}{(2n+1)(2n)!} \sim \frac{(2n)^{\alphahigh-\beta}}{(2n+1)\Gamma(\alphahigh-\beta)} \, . 
\end{equation*}
As a consequence, the serie $\sum_n n^{\alphahigh-\beta-1}\varphi_n(\beta)$ diverges. The criteria for the convergence of Riemann series implies that $\varphi_n(\beta)$ can not be asymptotically dominated by $1/n^{\alphahigh-\beta+\varepsilon}$ for $\varepsilon > 0$. 
\medskip 

We now turn to the case \ref{ass:reg-feature} where the features are irregular: with positive probability $p>0$, $X \notin \Sigma^{\alphahigh/2}(\cH)$ and $\langle X, \theta_* \rangle \neq 0$. With probability $p$, the second iterate $\theta_1 = -\gamma \langle X_1, \theta_* \rangle X_1$ is irregular, i.e., $\theta_1 \notin \Sigma^{\alphahigh/2}(\cH)$. By a simple shift of the iterates, we show that the effect of the irregularity of the initial condition for this iteration started from $\theta_1$ has an effect equivalent to the irregularity of the optimum, thus we can apply the result above to lower bound the convergence rate. More precisely, consider the iterates $\tilde{\theta}_n = \theta_{n+1}-\theta_1$ and $\tilde{\theta}_* = \theta_* - \theta_1$. The iteration \eqref{eq:SGD_iteration} can be rewritten as $\tilde{\theta}_n = \tilde{\theta}_{n-1} - \gamma \langle \tilde{\theta}_{n-1} - \tilde\theta_*, X_n \rangle X_n$ and $\tilde{\theta}_{0} = 0$, thus the new sequence $\tilde{\theta}_{n}$ satisfies our framework. We can assume that \ref{ass:reg-opt} is satisfied, i.e., $\theta_* \in \Sigma^{\alphahigh/2}(\cH)$. In that case, with probability $p$, $\tilde{\theta}_* = \theta_* - \theta_1 \notin \Sigma^{\alphahigh/2}(\cH)$. Thus by the case above, 
\begin{align*}
\varphi_n(\beta) &= \E\left[\left\langle \theta_n - \theta_*, \Sigma^{-\beta}\left(\theta_n - \theta_*\right)\right\rangle\right] \\
&=  \E\left[\left\langle \tilde{\theta}_{n-1} - \tilde{\theta}_*, \Sigma^{-\beta}\left(\tilde{\theta}_{n-1} - \tilde{\theta}_*\right)\right\rangle\right]
\end{align*}
is not asymptotically dominated by $1/n^{\alphahigh - \beta + \varepsilon}$, for $\varepsilon > 0$.

\section{Proof of Corollary \ref{coro:averaging}}
\label{sec:proof-averaging}

We apply Theorem \ref{thm:main-result-upper-bound} in the following way. Denote $\theta_n = x_n - x_0$, $\theta_* = x_*-x_0$, where $x_* = \frac{1}{N}\bfone$ is the function identically equal to $\frac{1}{N}$. These vectors belong to the Hilbert space $\cH = \ell^2(\cV)$. Denote $\langle ., .\rangle$ and $\Vert . \Vert$ the $\ell^2(\cV)$ scalar product and norm. Denote also $X_n = e_{v_n} - e_{w_n} \in \cH$ and $\gamma = 1/2$. Note that $\Sigma = \E[X_nX_n^\top] = \frac{1}{M}L$. The graph is connected thus $\lambda_0 = 0$ is the unique zero eigenvalue of $L$ \cite[Lemma 1.7]{chung1997spectral}. The corresponding eigenspace is the space of constant functions. The vectors $\theta_n, X_n, \theta_*$ are orthogonal to the null space of $\Sigma$, thus the quantities of the form $\langle \theta_n, \Sigma^{-\alpha} \theta_n \rangle$, $\langle X_n, \Sigma^{-\alpha} X_n \rangle$,$\langle \theta_*, \Sigma^{-\alpha} \theta_* \rangle$ are finite.  

We have $\theta_0 = 0$ and the averaging update step \eqref{eq:updates-averaging} can be written as
\begin{align*}
\theta_n &= \theta_{n-1} - \gamma \left\langle \theta_{n-1}-\theta_*, X_n\right\rangle X_n  \, .
\end{align*}
The last form makes explicit the parallel with Equation \eqref{eq:SGD_iteration}. To apply Theorem \ref{thm:main-result-upper-bound}, we check that its assumptions are satisfied. First, $\Vert X_n \Vert^2  = 2$ a.s.~thus can take $R_0=2$ and then $\gamma = 1/R_0$. Second, we seek $\alpha > 0$ such that $\Vert \Sigma^{-\alpha/2} \theta_* \Vert < \infty$ and $R_\alpha = \sup_{\{v,w\}\in \cE} \left\langle e_v - e_w, \Sigma^{-\alpha}(e_v-e_w) \right\rangle < \infty$. In the following, we bound these constants for all $\alpha < d/2$, thus giving decay rates for the expected squared distance to optimum of the form $n^{-\alpha}$ for all $\alpha < d/2$. However, our bounds of the constants $\Vert \Sigma^{-\alpha/2} \theta_* \Vert$ and $R_\alpha$ diverge as $\alpha \to d/2$. Nevertheless, by estimating how fast the bounds diverge as $\alpha \to d/2$, we obtain a decay rate of $n^{-d/2}$ by paying an additional logarithmic factor. 

Fix $0 < \alpha < d/2$. We check assumptions \ref{ass:reg-opt} and \ref{ass:reg-feature}. 
\begin{enumerate}[label = (\alph*)]
	\item \begin{align*}
	\Vert \Sigma^{-\alpha/2} \theta_* \Vert^2 &= M^\alpha \left\langle x_*-x_0 , L^{-\alpha} (x_*-x_0) \right\rangle = M^\alpha \sum_{i = 1}^{N-1} \lambda_i^{-\alpha} \left\langle x_* - x_0, u_i \right\rangle^2 \, .
	\end{align*}
	First, as $x_*$ is a constant vector, $\langle x_*, u_i \rangle$ is zero for all $i \geq 1$. Second, $x_0 = e_{v_\star}$. Thus 
	\begin{align*}
	\Vert \Sigma^{-\alpha/2} \theta_* \Vert^2
	&= M^\alpha \sum_{i = 1}^{N-1} \lambda_i^{-\alpha} u_i(v_\star)^2 \\ 
	&= M^\alpha \int_{(0,\infty)} \diff\sigma_{v_\star}(\lambda) \, \lambda^{-\alpha} \\
	&=  M^\alpha \int_{(0,\infty)} \diff\sigma_{v_\star}(\lambda) \int_{0}^{\infty} \diff s \, \bfone_{\{s\leq \lambda^{-\alpha}\}} \\
	&=   M^\alpha \int_{0}^{\infty} \diff s\int_{(0,\infty)} \diff\sigma_{v_\star}(\lambda) \,  \bfone_{\{\lambda \leq s^{-1/\alpha}\}} \\
	&=M^\alpha \int_{0}^{\infty} \diff s \, \sigma_{v_\star}((0,s^{-1/\alpha}]) \, .
	\end{align*}
	The graph $G$ is of spectral dimension $d$ with constant $V$, thus $\sigma_{v_\star}((0,s^{-1/\alpha}]) \leq V^{-1} s^{-\frac{d}{2\alpha}}$. However, if $s < \delta_{\max}^{-\alpha}$, it is better to use a more naive bound. As all eigenvalues of $L$ are smaller or equal than $\delta_{\max}$, $\sigma_{v_\star}((0,s^{-1/\alpha}]) \leq  \sigma_{v_\star}((0,\delta_{\max}]) \leq V^{-1}\delta_{\max}^{d/2}$. Then 
	\begin{align*}
	\Vert \Sigma^{-\alpha/2} \theta_* \Vert^2 &\leq M^\alpha \left[\int_{0}^{\delta_{\max}^{-\alpha}} \diff s \, V^{-1}\delta_{\max}^{d/2} +   \int_{\delta_{\max}^{-\alpha}}^\infty \diff s \, V^{-1}s^{-\frac{d}{2\alpha}}\right] \\
	&= M^\alpha V^{-1} \delta_{\max}^{d/2-\alpha} \frac{d}{d-2\alpha} \, .
	\end{align*}
	\item Let $\{v,w\}\in E$. As $\Vert \Sigma^{-\alpha/2} . \Vert$ is a norm, by the triangle inequality,
	\begin{align*}
	\Vert \Sigma^{-\alpha/2} (e_v-e_w) \Vert^2 &= \Vert \Sigma^{-\alpha/2} \left[(x_*-e_w) - (x_*-e_v)\right] \Vert^2 \\
	&\leq \left(\Vert \Sigma^{-\alpha/2}(x_*-e_w) \Vert + \Vert \Sigma^{-\alpha/2}(x_*-e_v) \Vert\right)^2 \\
	&\leq 2\left(\Vert \Sigma^{-\alpha/2}(x_*-e_w) \Vert^2 + \Vert \Sigma^{-\alpha/2}(x_*-e_v) \Vert^2\right) \, .
	\end{align*}
	We bound the two quantities as above. We obtain
	\begin{align*}
	R_\alpha = \sup_{{v,w}\in E }\Vert \Sigma^{-\alpha/2} (e_v-e_w) \Vert^2 &\leq 2 M^\alpha V^{-1} \delta_{\max}^{d/2-\alpha} \frac{d}{d-2\alpha} \, .
	\end{align*}
\end{enumerate}
Theorem \ref{thm:main-result-upper-bound} gives 
\begin{align*}
\E\left[\Vert x_n - x_* \Vert^2\right] &= \E\left[\Vert \theta_n - \theta_* \Vert^2\right] \leq \frac{\alpha^\alpha}{\gamma^{\alpha}} \left(\Vert \Sigma^{-\alpha/2} \theta_* \Vert^2 + \frac{R_\alpha}{R_0} \Vert  \theta_* \Vert^2\right) \frac{1}{n^\alpha} \\
&\leq \frac{(d/2)^\alpha}{(1/2)^\alpha} \left(M^\alpha V^{-1} \delta_{\max}^{d/2-\alpha} \frac{d}{d-2\alpha} + M^\alpha V^{-1} \delta_{\max}^{d/2-\alpha} \frac{d}{d-2\alpha} \Vert  \theta_* \Vert^2 \right) \frac{1}{n^\alpha} 
\end{align*}
Note that $\Vert \theta_* \Vert_2^2 \leq 1$ and recall the scaling $t=n/M$:
\begin{equation*}
\E\left[\Vert x_n - x_* \Vert^2\right] \leq d^{d/2+1}V^{-1} \delta_{\max}^{d/2-\alpha} \frac{1}{d/2-\alpha} \frac{1}{t^\alpha} \, .
\end{equation*}
This bound is valid for all $\alpha < \frac{d}{2}$. Choose $\alpha = \frac{d}{2} - \frac{\log 2}{\log t}$. 
\begin{equation*}
\E\left[\Vert x_n - x_* \Vert^2\right] \leq d^{d/2+1}V^{-1} \delta_{\max}^{\log 2/\log t } \frac{\log t }{\log 2} \frac{2}{t^{d/2}}
\end{equation*}
As we assume $t \geq 2$, $\delta_{\max}^{\log 2/\log t} \leq \delta_{\max}$. Thus we obtain conclusion \ref{concl:gossip-norm}. 

The proof of \ref{concl:gossip-energy} is similar. Theorem \ref{thm:main-result-upper-bound} gives 
\begin{align*}
\min_{0 \leq k \leq n} \E\left[\frac{1}{2}\sum_{\{v,w\}\in \cE} \left(x_k(v) -x_k(w) \right)^2\right] &= \min_{0 \leq k \leq n} \E\left[\frac{1}{2}\left\langle x_k - x_* , L(x_k - x_*)\right\rangle\right] \\
& = M \min_{0 \leq k \leq n} \E\left[\frac{1}{2}\left\langle \theta_k - \theta_* , \Sigma(\theta_k - \theta_*)\right\rangle\right] \\
&\leq 2^\alpha \frac{\alpha^\alpha}{\gamma^{\alpha+1}}\left(\Vert \Sigma^{-\alpha/2} \theta_* \Vert^2 + \frac{R_\alpha}{R_0} \Vert  \theta_* \Vert^2\right) \frac{1}{n^\alpha} \\
&\leq 2^{\alpha+1}d^\alpha V^{-1} \delta_{\max}^{d/2-\alpha} \frac{d}{d/2-\alpha} \frac{1}{t^{\alpha+1}} \, . 
\end{align*}
Taking again $\alpha = \frac{d}{2}-\frac{1}{2 \log t}$ and $t \geq 2$,
\begin{align*}
\min_{0 \leq k \leq n} \E\left[\frac{1}{2}\sum_{\{v,w\}\in \cE} \left(x_k(v) -x_k(w) \right)^2\right] &\leq 2^{d/2+1} d^{d/2} V^{-1} \delta_{\max} \frac{d \log t }{\log 2} \frac{2}{t^{d/2+1}}
\end{align*} 
This gives conclusion \ref{concl:gossip-energy} of the corollary.

\section{Proof of Proposition \ref{prop:torus}} 
\label{sec:proof-dim-torus}

The graph $\T^d_\Lambda$ is invariant by translation, thus the spectral measure $\sigma_v$ is the same for all vertices $v \in \cV$. Thus
\begin{align*}
\vert \cV \vert \sigma_v(\diff\lambda) = \sum_{w\in\cV} \sigma_w(\diff\lambda) = \sum_{w\in\cV} \sum_{i = 0}^{N-1} u_i(w)^2\delta_{\lambda_i} =  \sum_{i = 0}^{N-1} \left(\sum_{w\in\cV} u_i(w)^2 \right) \delta_{\lambda_i} = \sum_{i = 0}^{N-1} \delta_{\lambda_i} \, .
\end{align*}	
Thus 
\begin{align*}
\sigma_v((0,E]) = \frac{1}{\Lambda^d} \left\vert \left\{0 < i \leq N-1 \middle\vert \lambda_i \leq E \right\} \right\vert  \, .
\end{align*}
We need to bound the number of eigenvalues of the Laplacian of $\T^d_\Lambda$ below some fixed value $E$. The eigenvalues of the Laplacian of the circle $\T^1_\Lambda$ are $1-\cos\left(\frac{2\pi i}{\Lambda}\right)$, $i \in \Z, -\Lambda/2 < i \leq \Lambda/2$ \cite[Example 1.5]{chung1997spectral}. As $\T^d_\Lambda$ is the Cartesian product $\T^1_\Lambda \times \dots \times \T^1_\Lambda$ (with $d$ terms), the eigenvalues of the Laplacian of the torus $\T^d_\Lambda$ are the 
\begin{equation*}
1-\cos\left(\frac{2\pi i_1}{\Lambda}\right) + \dots + 1-\cos\left(\frac{2\pi i_d}{\Lambda}\right)\, , \qquad i_1, \dots i_d \in \Z, \quad  -\frac{\Lambda}{2} < i_1, \dots, i_d \leq \frac{\Lambda}{2} \, .
\end{equation*}
For $y \in [-\pi, \pi]$, $1-\cos(y) \geq \frac{2}{\pi^2}y^2$. Thus
\begin{align*}
1-\cos\left(\frac{2\pi i_1}{\Lambda}\right) + \dots + 1-\cos\left(\frac{2\pi i_d}{\Lambda}\right) \leq E &\Rightarrow \frac{2}{\pi^2} \left[\left(\frac{2\pi i_1}{\Lambda}\right)^2 + \dots + \left(\frac{2\pi i_d}{\Lambda}\right)^2 \right] \leq E \\
&\Leftrightarrow i_1^2 + \dots + i_d^2 \leq \frac{E\Lambda^2}{8} \, .
\end{align*}
We need to count the number of integer points in the Euclidean ball centered at $0$ and of radius $\sqrt{E/8}\Lambda$ in $\R^d$. This problem is famously known as Gauss circle problem. For our purposes, a crude estimate suffices: there exists a constant $C(d)$, depending only on the dimension $d$, such that for all radius $R$, the number of integer points in the ball of radius $R$ is smaller than $1 + C(d)R^d$. This leads to the final estimate
\begin{align*}
\sigma_v((0,E]) &= \frac{1}{\Lambda^d} \bigg\vert \bigg\{(i_1,\dots,i_d) \in \left(\Z \cap \left(-\frac{\Lambda}{2}, \frac{\Lambda}{2}\right]\right)^d \backslash \left\{0\right\} \text{ such that }\\
&\hspace{5cm} \, 1-\cos\left(\frac{2\pi i_1}{\Lambda}\right) + \dots + 1-\cos\left(\frac{2\pi i_d}{\Lambda}\right) \leq E \bigg\} \bigg\vert  \\
&\leq \frac{1}{\Lambda^d} \left\vert \left\{(i_1,\dots,i_d) \in \Z^d \backslash \left\{0\right\} \,  \middle\vert \,  i_1^2 + \dots + i_d^2 \leq \frac{E\Lambda^2}{8} \right\} \right\vert \\
&\leq \frac{1}{\Lambda^d} C(d) \left(\frac{E\Lambda^2}{8} \right)^{d/2} = \frac{C(d)}{8^{d/2}}E^{d/2} \, .
\end{align*}
This proves the proposition with $V(d) = 8^{d/2}/C(d)$. 

\section{Proof of Theorems \ref{thm:gen-main-result} and \ref{thm:gen-general-result}}
\label{sec:proof-thm-gen}

	Note that in this proof, we use the strong assumptions of regularity of the feature vector $X$. We do not know whether it is possible to prove the same result under the weak assumptions of Remark \ref{rmk:weak-assumptions}. 
	
	Our proof stategy is the following: we decompose the SGD iterates sequence $\theta_n$ as a sum of sequences $\theta_n = \nu_n + \sum_{l=1}^{n}\eta_n^{(l)}$, where each of the auxiliary sequences is interpreted as the iterates of some SGD iteration under a noiseless linear model. We thus apply the results of Section \ref{sec:general-theory} to control these auxiliary sequences and obtain the presented bound. 
	
	Define $\varepsilon_n = Y_n - \langle \theta_*, X_n \rangle$, the error of the best linear estimator. Then Equation \eqref{eq:gen-sgd-iteration} can be rewritten as
	\begin{align*}
		&\theta_{0} = 0 \, , &\theta_n = \theta_{n-1} - \gamma \langle \theta_{n-1} - \theta_*, X_n \rangle X_n + \gamma \varepsilon_n X_n \, .
	\end{align*}
	We see this iteration as an additively perturbed version of the iteration 
	\begin{align*}
		&\nu_{0} = 0 \, , &\nu_n = \nu_{n-1} - \gamma \langle \nu_{n-1} - \theta_*, X_n \rangle X_n  \, ,
	\end{align*}
	studied in Section \ref{sec:general-theory}. To understand the effect of the additive noise, define for all $l \geq 1$, 
	\begin{align*}
		&\eta_l^{(l)} = \gamma \varepsilon_l X_l \, , &&\eta^{(l)}_{n} = \eta^{(l)}_{n-1} - \gamma \langle \eta^{(l)}_{n-1}, X_n \rangle X_n \, , \qquad n>l \, .
	\end{align*}
	Then 
	\begin{equation}
		\label{eq:aux-12}
		\theta_n = \nu_n + \sum_{l=1}^{n} \eta_n^{(l)} \, .
	\end{equation}
	Indeed, this last equation is checked by induction: $\theta_0 = 0 = \nu_0$, and if the equation is satisfied for some $n \geq 0$, 
	\begin{align*}
		\theta_{n+1} &= \theta_{n} - \gamma \langle \theta_{n} - \theta_*, X_{n+1} \rangle X_{n+1} + \gamma \varepsilon_{n+1} X_{n+1} \\
		&= \nu_n + \sum_{l=1}^{n} \eta_n^{(l)} - \gamma \left\langle \nu_n + \sum_{l=1}^{n} \eta_n^{(l)} - \theta_*, X_{n+1} \right\rangle X_{n+1} + \eta_{n+1}^{(n+1)} \\ 
		&= \left[\nu_n - \gamma \langle \nu_n - \theta_*, X_{n+1} \rangle X_{n+1} \right] + \sum_{l=1}^{n}\left[\eta_n^{(l)} - \gamma \langle \eta_n^{(l)}, X_{n+1} \rangle X_{n+1} \right] + \eta_{n+1}^{(n+1)} \\
		&= \nu_{n+1} + \sum_{l=1}^{n} \eta_{n+1}^{(l)} + \eta_{n+1}^{(n+1)} \, .
	\end{align*}
	We use the decomposition \eqref{eq:aux-12} to study $\varphi_n(\beta)$. Using the triangle inequality,
	\begin{align}
		\varphi_n(\beta) &= \E\left[\left\Vert \Sigma^{-\beta/2} \left(\nu_n + \sum_{l=1}^{n}\eta_n^{(l)}\right)\right\Vert^2\right] \nonumber \\
		&\leq \E\left[\left(\left\Vert \Sigma^{-\beta/2} \nu_n \right\Vert + \left\Vert \Sigma^{-\beta/2} \sum_{l=1}^{n}\eta_n^{(l)} \right\Vert\right)^2\right] \nonumber \\
		&\leq 2\E\left[\left\Vert \Sigma^{-\beta/2} \nu_n \right\Vert^2\right] + 2\E\left[ \left\Vert \Sigma^{-\beta/2} \sum_{l=1}^{n}\eta_n^{(l)} \right\Vert^2\right] \label{eq:aux-14}
	\end{align}
	The first term is studied in Section \ref{sec:general-theory}. We detail the analysis of the second term. Note that 
	\begin{align}
		\eta_n^{(l)} &= (I-\gamma X_n \otimes X_n) \eta_{n-1}^{(l)} = \dots = (I-\gamma X_n \otimes X_n) \cdots (I-\gamma X_{l+1} \otimes X_{l+1}) \eta_l^{(l)} \nonumber \\
		&= (I-\gamma X_n \otimes X_n) \cdots (I-\gamma X_{l+1} \otimes X_{l+1}) \gamma \varepsilon_l X_l \label{eq:aux-13}\, . 
	\end{align}
	Thus if $l < l'$, 
	\begin{align*}
		\E\left[\left\langle \eta_n^{(l)}, \Sigma^{-\beta} \eta_n^{(l')}\right\rangle\right] &= \E\left[\left\langle \E\left[\eta_n^{(l)} \middle\vert X_{l+1}, \dots, X_n\right], \Sigma^{-\beta} \eta_n^{(l')}\right\rangle\right] \\
		&= \E\left[\left\langle (I-\gamma X_n \otimes X_n) \cdots (I-\gamma X_{l+1} \otimes X_{l+1}) \gamma \E[\varepsilon_l X_l], \Sigma^{-\beta} \eta_n^{(l')}\right\rangle\right]
	\end{align*}
	Note that by definition of $\theta_*$, $0 = \nabla \cR(\theta_*) = -\E\left[(Y_l-\langle \theta_*, X_l \rangle)X_l\right] = -\E\left[\varepsilon_l X_l\right]$ thus we obtain that the cross products $\E\left[\left\langle \eta_n^{(l)}, \Sigma^{-\beta} \eta_n^{(l')}\right\rangle\right]$ are zero. This gives 
	\begin{align*}
		\E\left[ \left\Vert \Sigma^{-\beta/2} \sum_{l=1}^{n}\eta_n^{(l)} \right\Vert^2\right] = \sum_{l=1}^{n} \E\left[ \left\Vert \Sigma^{-\beta/2} \eta_n^{(l)} \right\Vert^2\right] \, .
	\end{align*}
	Note that from Equation \eqref{eq:aux-13}, $\eta_n^{(l)}$ and $\eta_{n-l+1}^{(1)}$ are equal in law. Thus 
	\begin{align}
		\E\left[\left\Vert \Sigma^{-\beta/2} \sum_{l=1}^{n}\eta_n^{(l)} \right\Vert^2\right] = \sum_{l=1}^{n} \E\left[ \left\Vert \Sigma^{-\beta/2} \eta_{n-l+1}^{(1)} \right\Vert^2\right] =  \sum_{l=1}^{n} \E\left[ \left\Vert \Sigma^{-\beta/2} \eta_{l}^{(1)} \right\Vert^2\right] \, . \label{eq:aux-15}
	\end{align}
	This last quantity is the sum of the expected squared power norms \begin{equation*}
		\varphi_l'(\beta) :=  \E\left[ \left\Vert \Sigma^{-\beta/2} \eta_{l}^{(1)} \right\Vert^2\right] 
	\end{equation*}
	of the SGD iterates $\eta^{(1)}_l, l \geq 1$ on a noiseless linear model, with initialization $\eta^{(1)}_1 = \gamma \varepsilon_1 X_1$. When $\beta = -1$, this control is given by \eqref{eq:bound_sum_-1}: with our notation here, this gives
	\begin{equation}
		\sum_{l=1}^n \varphi'_l(-1) \leq \sum_{l=1}^\infty \varphi'_l(-1) \leq \frac{1}{\gamma} \varphi'_1(0) \, . \label{eq:aux-8}
	\end{equation}
	When $\beta = \alphalow-1$, a similar control can be obtained from \eqref{eq:aux-7} which gives:
	\begin{equation*}
		2\gamma \varphi'_{l-1}(\alphalow-1) \leq \varphi'_{l-1}(\alphalow) - \varphi'_{l}(\alphalow) + \gamma^2 R_\alphalow \varphi'_{l-1}(-1) \, . 
	\end{equation*}
	By summing these inequalities for $l=2, 3, \dots$, we obtain,
	\begin{align}
		2\gamma \sum_{l=1}^\infty \varphi'_l(\alphalow-1) &\leq \varphi'_1(\alphalow) + \gamma^2 R_\alphalow \sum_{l=1}^{\infty} \varphi'_l(-1) \nonumber\\
		&\leq \varphi'_1(\alphalow) + \frac{R_\alphalow}{R_0} \varphi'_1(0) \label{eq:aux-9} 
	\end{align}
	Note that using the strong assumption of regularity of the feature vectors, 
	\begin{align*}
		&\varphi'_1(0) = \E\left[ \left\Vert \gamma \varepsilon_1 X_1 \right\Vert^2\right] \leq \gamma^2 R_0 \E\left[\varepsilon_1^2\right] = 2\gamma^2 R_0 \cR(\theta_*) \, , \\
		&\varphi'_1(\alphalow) =  \E\left[ \left\Vert \Sigma^{-\alphalow/2}  \gamma \varepsilon_1^2 X \right\Vert^2\right] \leq \gamma^2  R_\alphalow \E\left[\varepsilon_1^2\right] = 2 \gamma^2  R_\alphalow \cR(\theta_*) \, .
	\end{align*} 
	We use these expressions to simply further \eqref{eq:aux-8} and \eqref{eq:aux-9}:
	\begin{align*}
		\sum_{l=1}^n \varphi'_l(-1) &\leq 2\gamma R_0 \cR(\theta_*) \, , \\
		\sum_{l=1}^\infty \varphi'_l(\alphalow-1) &\leq 2\gamma R_\alphalow \cR(\theta_*) \, . 
	\end{align*}
	If $\beta \in [-1,\alphalow-1]$, we use the log-convexity Property \ref{prop:log-cvx} and H\"older's inequality: decompose $\beta = (1-\lambda)(-1) + \lambda(\alphalow-1)$ with $\lambda = (\beta+1)/\alphalow$,
	\begin{align}
		\sum_{l=1}^\infty \varphi'_l(\beta) &\leq \sum_{l=1}^\infty \varphi'_l(-1)^{1-\lambda} \varphi'_l(\alphalow-1)^{\lambda} \nonumber \\
		&\leq \left(\sum_{l=1}^n \varphi'_l(-1)\right)^{1-\lambda} \left(\sum_{l=1}^\infty \varphi'_l(\alphalow-1)\right)^\lambda \nonumber\\
		&\leq \left(2\gamma R_0 \cR(\theta_*)\right)^{1-\lambda} \left(2\gamma R_\alphalow \cR(\theta_*)\right)^\lambda \nonumber\\
		&= 2\gamma R_0^{1-\lambda} R_\alphalow^\lambda \cR(\theta_* ) \label{eq:aux-10}\, .
	\end{align}
	
	Putting back together Equations \eqref{eq:aux-14}, \eqref{eq:aux-15} and \eqref{eq:aux-10}, we obtain 
	\begin{align*}
		\varphi_n(\beta) \leq 2\E\left[\left\Vert \Sigma^{-\beta/2} \nu_n \right\Vert^2\right] + 4\gamma R_0^{1-\lambda} R_\alphalow^\lambda \cR(\theta_* ) 
	\end{align*}
	The theorem follows the application of Theorem \ref{thm:general-result-upper-bound} to the sequence $\nu_n$ in order to control the first term.

\end{document}